\documentclass{article}
\usepackage[dvipsnames]{xcolor}
\usepackage{tablefootnote}
\usepackage{amsmath}
\usepackage{amssymb}
\usepackage{mathtools}
\usepackage{hyperref}
\usepackage{cleveref}
\usepackage{tikz}
\usepackage{tikz-qtree}
\usetikzlibrary{trees} %
\usepackage[numbers]{natbib}
\usepackage[preprint]{arxiv}
\usepackage{equations}
\usepackage{wrapfig}

\def\eqref#1{(\ref{#1})}
\usepackage[utf8]{inputenc} %
\usepackage[T1]{fontenc}    %
\usepackage{url}            %
\usepackage{booktabs}       %
\usepackage{amsfonts}       %
\usepackage{mathtools}
\usepackage{enumerate}
\usepackage{nicefrac}       %
\usepackage{microtype}      %
\usepackage{subcaption}
\usepackage{multirow} %
\usepackage{caption} %
\usepackage{hhline} %

\usepackage{arydshln}
\usepackage{enumitem}
\usepackage[normalem]{ulem}

\usepackage{tabu}
\usepackage[edges]{forest}
\usepackage{graphicx}
\setitemize{noitemsep,topsep=0.pt,parsep=0pt,partopsep=0pt}

\usepackage{amsmath,amssymb,amsfonts,bm,mathtools}
\usepackage{amsthm}
\usepackage{graphicx}

\newcommand{\dx}{\mathrm{d}\vx}
\newcommand{\dt}{\mathrm{d}t}

\newcommand{\dW}{\mathrm{d}\vW}

\newcommand{\vepsilon}{\bm{\epsilon}}

\renewcommand{\d}{\mathrm{d}}
\theoremstyle{plain}
\newtheorem{theorem}{Theorem}[section]

\newtheorem{lemma}[theorem]{Lemma}
\newtheorem{corollary}[theorem]{Corollary}
\theoremstyle{definition}

\theoremstyle{remark}

\newtheorem{example}{Example}[section]

\usepackage[ruled,linesnumbered]{algorithm2e}
\SetKwComment{Comment}{/* }{ */}
\usepackage{wasysym}
\usepackage{mathtools}
\usepackage{authblk}
\usepackage{algorithmic}
\usepackage{subcaption}
\usepackage{wrapfig}
\usepackage{resizegather}

\def\eqref#1{equation~\ref{#1}}
\def\Eqref#1{Equation~\ref{#1}}

\numberwithin{equation}{section}

\def\1{\bm{1}}

\def\vmu{{\bm{\mu}}}

\def\vphi{{\bm{\phi}}}
\def\vPhi{{\bm{\Phi}}}

\def\vc{{\bm{c}}}

\def\vf{{\bm{f}}}

\def\vh{{\bm{h}}}

\def\vn{{\bm{n}}}

\def\vr{{\bm{r}}}
\def\vs{{\bm{s}}}

\def\vW{{\bm{W}}}
\def\vx{{\bm{x}}}
\def\vX{{\bm{X}}}
\def\vy{{\bm{y}}}
\def\vY{{\bm{Y}}}
\def\vz{{\bm{z}}}
\def\vZ{{\bm{Z}}}

\DeclareMathAlphabet{\mathsfit}{\encodingdefault}{\sfdefault}{m}{sl}
\SetMathAlphabet{\mathsfit}{bold}{\encodingdefault}{\sfdefault}{bx}{n}

\def\gN{{\mathcal{N}}}

\def\gT{{\mathcal{T}}}

\newcommand{\E}{\mathbb{E}}

\newcommand{\R}{\mathbb{R}}

\DeclareMathOperator*{\argmin}{arg\,min}

\newcommand{\norm}[1]{\left\|#1\right\|}

\newcommand{\net}[1]{\vh_{\theta}(#1)}

\renewcommand{\norm}[1]{\left|\left| #1 \right|\right|^2}

\usepackage{pifont}%
\newcommand{\cmark}{\ding{51}}%
\newcommand{\xmark}{\ding{55}}%

\definecolor{bluebell}{rgb}{0.64, 0.64, 0.82}


\title{A Survey on Diffusion Models for Inverse Problems }

\author{
    Giannis Daras\textsuperscript{1},
    Hyungjin Chung\textsuperscript{2},
    Chieh-Hsin Lai\textsuperscript{3}, 
    Yuki Mitsufuji\textsuperscript{3},\\
    Jong Chul Ye\textsuperscript{2},
    Peyman Milanfar\textsuperscript{4}, 
    Alexandros G. Dimakis\textsuperscript{1},
    Mauricio Delbracio\textsuperscript{4}
}

\vspace{1em}

\affil{
    \textsuperscript{1}UT Austin \quad
    \textsuperscript{2}KAIST \quad
    \textsuperscript{3}Sony AI \quad
    \textsuperscript{4}Google
}
\date{}

\begin{document}

\maketitle

\begin{abstract}
Diffusion models have become increasingly popular for generative modeling due to their ability to generate high-quality samples. This has unlocked exciting new possibilities for solving inverse problems, especially in image restoration and reconstruction, by treating diffusion models as unsupervised priors. This survey provides a comprehensive overview of methods that utilize pre-trained diffusion models to solve inverse problems without requiring further training. We introduce taxonomies to categorize these methods based on both the problems they address and the techniques they employ.  We analyze the connections between different approaches, offering insights into their practical implementation and highlighting important considerations.  We further discuss specific challenges and potential solutions associated with using latent diffusion models for inverse problems. This work aims to be a valuable resource for those interested in learning about the intersection of diffusion models and inverse problems.

\end{abstract}

\section{Introduction}
\begin{table}[!htb]
\centering
\resizebox{1.0\textwidth}{!}{%
\begin{tabular}{lcccccccc}
\toprule
{\textbf{Category}} & {\textbf{Method}} & \textbf{Non-linear} & \textbf{Blind} & \textbf{Handle noise} & \textbf{Pixel/Latent} & \textbf{Text-conditioned} & \textbf{Optimization Technique} & \textbf{Code}\textsuperscript{1}
\\
\midrule
\multirow{15}{*}{\rotatebox{90}{\parbox{4cm}{\centering Explicit approximations for measurement matching}}} & Score-ALD~\citet{mri_paper} & \xmark & \xmark & \cmark & Pixel & \xmark & Grad & \href{https://github.com/utcsilab/csgm-mri-langevin}{code}\\
 & Score-SDE~\citet{ncsnv3} & \xmark & \xmark & \xmark & Pixel & \xmark & Proj & \href{https://github.com/yang-song/score_sde}{code}\\
 & ILVR~\citet{ilvr} & \xmark & \xmark & \xmark & Pixel & \xmark & Proj & \href{https://github.com/jychoi118/ilvr_adm}{code}\\
 & DPS~\citet{dps} & \cmark & \xmark & \cmark & Pixel & \xmark & Grad & \href{https://github.com/DPS2022/diffusion-posterior-sampling}{code}\\
 & $\Pi$GDM~\citet{song2022pseudoinverse} & \cmark & \xmark & \cmark & Pixel & \xmark & Grad & \href{https://github.com/NVlabs/RED-diff}{code}\\
 & Moment Matching~\citet{rozet2024learning} & \cmark & \xmark & \cmark & Pixel & \xmark & Grad & \href{https://github.com/francois-rozet/diffusion-priors}{code}\\
 & BlindDPS~\citet{blinddps} & \cmark & \cmark & \cmark & Pixel & \xmark & Grad & \href{https://github.com/BlindDPS/blind-dps}{code}\\
 & SNIPS~\citet{kawar2021snips} & \xmark & \xmark & \cmark & Pixel & \xmark & Grad & \href{https://github.com/bahjat-kawar/snips_torch}{code}\\
 & DDRM~\citet{ddrm} & \xmark & \xmark & \cmark & Pixel & \xmark & Grad & \href{https://github.com/bahjat-kawar/ddrm}{code}\\
 & GibbsDDRM~\citet{murata2023gibbsddrm} & \xmark & \cmark & \cmark & Grad & \xmark & Samp & \href{https://github.com/sony/gibbsddrm}{code}\\
 & DDNM~\citet{ddnm} & \xmark & \xmark & \cmark & Pixel & \xmark & Proj & \href{https://github.com/wyhuai/DDNM}{code}\\
 & DDS~\citet{dds} & \xmark & \xmark & \cmark & Pixel & \xmark & Opt & \href{https://github.com/HJ-harry/DDS}{code}\\
 & DiffPIR~\citet{diffpir} & \xmark & \xmark & \cmark & Pixel & \xmark & Opt & \href{https://github.com/yuanzhi-zhu/DiffPIR}{code}\\
 & PSLD~\citet{rout2024solving} & \cmark & \xmark & \cmark & Latent & \xmark & Grad & \href{https://github.com/LituRout/PSLD}{code}\\
 & STSL~\citet{rout2024beyond} & \cmark & \xmark & \cmark & Latent & \xmark & Grad & \xmark \\
\cmidrule{1-9}
\multirow{4}{*}{\rotatebox{90}{\parbox{1.5cm}{\centering Variational inference}}} & RED-Diff~\citet{mardani2024variational} & \cmark & \xmark & \cmark & Pixel & \xmark & Opt & \href{https://github.com/NVlabs/RED-diff}{code}\\
 & Blind RED-Diff~\citet{alkan2023variational} & \cmark & \cmark & \cmark & Pixel & \xmark & Opt & \xmark \\
 & Score Prior~\citet{feng2023score} & \cmark & \xmark & \cmark & Pixel & \xmark & Opt & \href{https://github.com/berthyf96/score_prior}{code}\\
 & Efficient Score Prior~\citet{feng2023efficient} & \cmark & \xmark & \cmark & Pixel & \xmark & Opt & \href{https://github.com/berthyf96/score_prior}{code}\\
\cmidrule{1-9}
\multirow{4}{*}{\rotatebox{90}{\parbox{1.2cm}{\centering CSGM methods}}} 
& DMPlug~\citet{wang2024dmplug} & \cmark & \xmark & \cmark & Pixel & \xmark & Opt & \href{https://github.com/sun-umn/DMPlug}{code}\\
& SHRED~\citet{chihaoui2024zeroshot} & \cmark & \xmark & \cmark & Pixel & \xmark & Opt & \xmark\\
& Consistent-CSGM~\citep{xu2024consistency} & \cmark & \xmark & \cmark & Pixel & \xmark & Opt & \xmark\\ 
& Score-ILO~\citet{score_ilo} & \cmark & \xmark & \cmark & Pixel & \xmark & Opt & \href{https://github.com/giannisdaras/sgilo}{code} \\
\cmidrule{1-9} 
\multirow{5}{*}{\rotatebox{90}{\parbox{2.4cm}{\centering Asymptotically Exact Methods}}} 
 & PnP-DM~\citet{wu2024principled} & \cmark & \xmark & \cmark & Pixel & \xmark & Opt & \xmark\\
 & FPS~\citet{dou2023diffusion} & \cmark & \xmark & \cmark & Pixel & \xmark & Samp & \href{https://github.com/ZehaoDou-official/FPS-SMC-2023}{code}\\
 & PMC~\citet{sun2024provable} & \cmark & \xmark & \cmark & Pixel & \xmark & Samp & \href{https://github.com/sunyumark/PnP-MonteCarlo}{code}\\
 & SMCDiff~\citet{trippe2023diffusion} & \xmark & \xmark & \cmark & Pixel & \xmark & Samp & \href{https://github.com/blt2114/ProtDiff_SMCDiff}{code}\\
 & MCGDiff~\citet{cardoso2023monte} & \xmark & \xmark & \cmark & Pixel & \xmark & Samp & \href{https://github.com/gabrielvc/mcg_diff}{code}\\
 & TDS~\citet{wu2023practical} & \xmark & \xmark & \cmark & Pixel & \xmark & Samp & \href{https://github.com/blt2114/twisted_diffusion_sampler}{code}\\
 \cmidrule{1-9}
\multirow{6}{*}{\rotatebox{90}{\parbox{1.5cm}{\centering Other methods}}} & Implicit denoiser prior~\citet{kadkhodaie2020solving} & \xmark & \xmark & \xmark & Pixel & \xmark & Proj & \href{https://github.com/LabForComputationalVision/universal_inverse_problem}{code}\\
& MCG~\citet{mcg} & \xmark & \xmark & \xmark & Pixel & \xmark & Grad/Proj & \href{https://github.com/HJ-harry/MCG_diffusion}{code}\\
& Resample~\citet{song2024solving} & \cmark & \xmark & \cmark & Latent & \xmark & Grad/Opt & \href{https://github.com/soominkwon/resample}{code}\\
& MPGD~\citet{he2023manifold} & \cmark & \xmark & \cmark & Pixel/Latent & \cmark & Grad/Opt & \href{https://github.com/KellyYutongHe/mpgd_pytorch}{code}\\
 & P2L~\citet{chung2024prompt} & \cmark & \xmark & \cmark & Latent & \cmark & Grad/Opt & \xmark \\
 & TReg~\citet{kim2023regularization} & \cmark & \xmark & \cmark & Latent & \cmark & Grad/Opt & \xmark \\
 & DreamSampler~\citet{kim2024dreamsampler} & \cmark & \xmark & \cmark & Latent & \cmark & Grad/Opt & \href{https://github.com/DreamSampler/dream-sampler}{code} \\
\bottomrule
\end{tabular}
}
\vspace{0.5cm}
\caption{\textbf{Categorization of Diffusion-Based Inverse Problem Solvers.}  This table categorizes methods by their approach to solving inverse problems with diffusion models. We identified four families of methods. \textit{Explicit Approximations for Measurement Matching:} These methods approximate the measurement matching score, $\nabla \log p_t(\vy|\vx_t)$, with a closed-form expression. \textit{Variational Inference:} These methods approximate the true posterior distribution, $p(\vx | \vy)$, with a simpler, tractable distribution. Variational formulations are then used to optimize the parameters of this simpler distribution. \textit{CSGM-type methods:} The works in this category use backpropagation to change the initial noise of the deterministic diffusion sampler, essentially optimizing over a latent space for the diffusion model.
\textit{Asymptotically Exact Methods:} These methods aim to sample from the true posterior distribution. This is typically achieved by constructing Markov chains (MCMC) or by propagating particles through a sequence of distributions (SMC) to obtain samples that approximate the posterior. 
Further categorization is based on being able to address non-linear problems, blind formulations (unknown forward model), noise handling, pixel/latent space operation, text-conditioning, and the type of optimization technique used (gradient-based, projection, etc.). Code availability is also indicated.}

\label{tab:categorization}
\end{table}

\begin{figure}[h]
\begin{tikzpicture}[grow'=right,level distance=2.5in,sibling distance=.15in]
\tikzset{level 1/.style={level distance=3.7in},
        edge from parent/.style= 
            {thick, draw, minimum width=0.5in, text width=0.5in, edge from parent fork right, shorten <=0.1pt},
         every tree node/.style=
            {draw, rounded corners, 
            minimum width=3.2in,
            text width=3.2in, 
            align=center},
        myroot/.style={draw, minimum width=1.5in, text width=1.5in, minimum height=1.5cm},     
        mychild/.style={draw, 
            rounded corners, 
            minimum width=3.2in,
            text width=3.2in, 
            align=center},
            }
\Tree 
    [.\node[myroot]{$\nabla_{\vx_t} \log p(\vy | \vX_t=\vx_t) $}; 
        [.{Score ALD: \begin{align*}
            \propto - \eqnmarkbox[Plum]{}{A^{\top}}\eqnmarkbox[NavyBlue]{}{\left( \vy - A\vx_t\right)}
        \end{align*}}
        ]
        [.{Score-SDE:
        \begin{align*}
           \propto - \eqnmarkbox[Plum]{}{A^{\top}}(\eqnmarkbox[RoyalPurple]{}{\vy_t - A\vx_t})
        \end{align*}
        }
        ] 
        [.{ILVR:
        \begin{align*}
          \propto   -\eqnmarkbox[Plum]{}{(A^\top A)^{-1} A^\top} (\eqnmarkbox[NavyBlue]{}{\vy_t - A\vx_t)}
        \end{align*}
        } ]
        [.{DPS:
        \begin{align*}
           \propto &\eqnmarkbox[Plum]{}{\left(I + \nabla_{\vx_t}^2\log p_{t}(\vx_t) \right)^{\top}A^{\top}} \\&\cdot\left(\eqnmarkbox[NavyBlue]{}{\vy - A\E[\vX_0 | \vX_t=\vx_t]}\right)
        \end{align*}
        } ]
        [.{$\Pi$GDM:
        \begin{align*}
           \propto  -&\eqnmarkbox[Plum]{}{\frac{\partial \E[\vX_0|\vX_t = \vx_t]}{\partial \vx_t}  (r_t^2 AA^\top + \sigma_{\vy}^2 I)^{-1} A^\top} \\&\cdot(\eqnmarkbox[NavyBlue]{}{\vy - A \E[\vX_0|\vX_t = \vx_t]})
        \end{align*}
        } ]
        [.\node[mychild]{Moment Matching:
        \begin{align*}
           \propto  -&\eqnmarkbox[Plum]{}{\nabla_{\vx_t}\E[\vx_0 | \vx_t]^{\top}A^{\top}(\sigma_{\vy}^2 I + A\sigma_t^2 \nabla_{\vx_t}\E[\vx_0 | \vx_t]A^{\top})^{-1}}\\&\cdot(\eqnmarkbox[NavyBlue]{}{\vy - A\E[\vx_0 | \vx_t]})
        \end{align*}
        }; ]
        [.{SNIPS:
        \begin{align*}
            \propto 
    -\eqnmarkbox[Plum]{proj}{\Sigma^\top \left|\sigma_{\vy}^2 I_m - \sigma_t^2\Sigma\Sigma^\top\right|^\dagger}(
    \eqnmarkbox[NavyBlue]{merror}{\bar\vy - \Sigma\bar\vx_t})
        \end{align*}
        } ]
        [.{DDRM:
        \begin{align*}
            \propto 
    -\eqnmarkbox[Plum]{}{\Sigma^\top \left|\sigma_{\vy}^2 I_m - \sigma_t^2\Sigma\Sigma^\top\right|^\dagger}
    (
    \eqnmarkbox[NavyBlue]{}{\bar\vy - \Sigma\bar\vx_{0|t}}
    )
        \end{align*}
        } ]
        [.{DDNM:
        \begin{align*}
        \propto
    \eqnmarkbox[Plum]{}{\Sigma_t A^\dagger}
    \left(
    \eqnmarkbox[NavyBlue]{}{\vy - A\E[\vX_0 | \vX_t=\vx_t, \vy]}
    \right)
        \end{align*}
        }
        ]
    ]
\end{tikzpicture}
\caption{Approximations for the measurements score proposed by different methods.}
\label{fig:tree}
\end{figure}
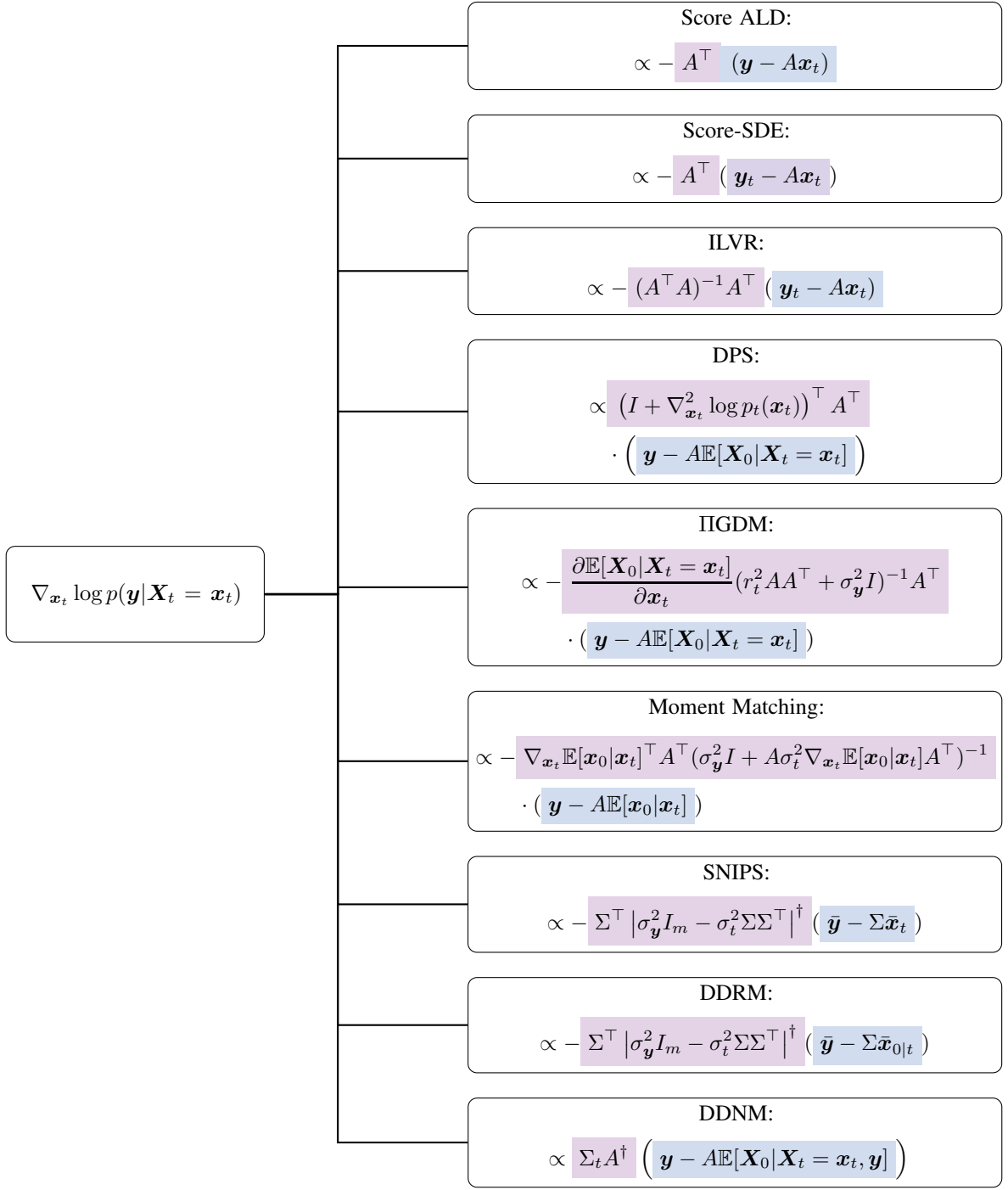

\subsection{Problem Setting}

Inverse problems are ubiquitous and the associated reconstruction problems have tremendous applications across different domains such as seismic imaging~\citep{lailly1983seismic, virieux2009overview}, weather prediction~\citep{huang2005inverse}, oceanography~\citep{wunsch1996ocean}, audio signal processing~\citep{lemercier2024diffusion,saito2023unsupervised,moliner2023blind, moliner2023solving,moliner2023diffusion,hernandez2024vrdmg},  medical imaging~\citep{song2021solving,chung2023solving,aali2024ambient,chung2022score}, etc. Despite their generality, inverse problems across different domains follow a fairly unified mathematical setting. Specifically, in inverse problems, the goal is to recover an unknown sample $\vx \in \R^n$ from a distribution $p_{\vX}$, assuming access to measurements $\vy \in \R^m$ and a corruption model 
\begin{align}\label{eq:problem}
    \vY = \mathcal A(\vX) + \sigma_\vy \vZ, \ \vZ \sim \mathcal N(\bm{0}, I_m).
\end{align}
In what follows, we present some well-known examples of measurement models that fit under this general formulation.

\begin{example}[Denoising] The simplest interesting example is the denoising inverse problem, i.e. when $A$ is the identity matrix and $\sigma_{\vy} > 0$. In fact, the noise model does not have to be Gaussian and it can be generalized to other distributions,  including the Laplacian Distribution or the Poisson Distribution~\citep{fan2019brief}. For the purposes of this survey, we focus on additive Gaussian noise. \end{example}

A lot of practical applications arise from the non-invertible linear setting, i.e. for $\mathcal A(\vX) = A\vX$ and $A$ being an $m\times n$ matrix with $m < n$. 
\begin{example}[Inpainting] 
$A$ is a masking matrix, i.e. $A_{ij} = 0$ for $i\neq j$ and $A_{ii}$ is either 0 or 1, based on whether the value at this location is observed.
\end{example}

\begin{example}[Compressed Sensing] 
$A$ is a matrix with entries sampled from a Gaussian random variable.
\end{example}

\begin{example}[Convolutions] Here $\mathcal A(\vX)$ represents the convolution of $\vX$ with a (Gaussian or other) kernel, which is again a linear operation. 
\end{example}

The same inverse problem can appear across vastly different scientific fields. To illustrate this point, we can take the inpainting case as an example. 
In Computer Vision, inpainting can be useful for applications such as object removal or object replacement~\citep{quan2024deep, rout2024solving, yu2023inpaint}. In the proteins domain, inpainting can be useful for protein engineering, e.g. by mutating certain aminoacids of the protein sequence to achieve better thermodynamical properties~\citep{ouyangzhang2023predicting, diaz2024stability, yang2019machine, xu2020deep}. MRI acceleration is also an inpainting problem but in the Fourier domain~\citep{aali2023solving, zbontar2018fastmri, desai2022skmtea, mridata_abdomens, mridata_knees}. Particularly,  for each coil measurement $y_i$ within the multi-coil setting, we have $A_i = P F S_i $, where $P$ is the masking operator, $F$ is the 2D discrete Fourier transform, and $S_i$ denotes the element-wise sensitivity value. For single-coil, $S_i$ is the identity matrix~\citet{lustig2008compressed}. Similarly, CT can be considered an inpainting problem in the Radon-transformed domain $A = P R$, where $R$ is the Radon transform~\citep{pan2009commercial, genzel2022near, beylkin1984inversion}. Depending on the circumstances such as sparse-view or limited-angle, the pattern of the masking operator $P$ differs~\citet{kak2001principles}. Finally, in the audio domain, the bandwidth extension problem, i.e. the task of recovering high-frequency content from an observed signal, is another example of inpainting in the spectrogram domain)~\citet{dietz2002spectral}.  

Inpainting is just one of many useful linear inverse problems in scientific applications and there are plenty of other important examples to consider. Cryo-EM~\citet{dubochet1988cryo} is a blind inverse problem that is defined by $A = CSR$, where $C$ is a blur kernel and $S$ is a shifting matrix, i.e. additional (unknown) shift and blur is applied to the projections.
Deconvolution appears in several applications such as super-resolution~\citep{park2003super, saharia2021image} of images and removing reverberant corruption~\citep{nakatani2010speech} in audio signals.

There are many interesting non-linear inverse problems too, i.e. where $\mathcal A$ is a nonlinear operator.

\begin{example}[Phase Retrieval~\citet{fienup1982phase}] Phase retrieval considers the nonlinear operator $\mathcal A(\vX):=|F \vX|$, where the measurement contains only the magnitude of the Fourier signal.
\end{example}
\begin{example}[Compression Removal] Here $\mathcal A(\vX;\alpha)$ represents a (non-linear) compression operator (e.g., JPEG) whose strength is controlled by the parameter $\alpha$.
\end{example}
A famous non-linear inverse problem is the problem of imaging a black hole, where the relationship between the image to be reconstructed and the interferometric measurement can be considered as a sparse and noisy Fourier phase retrieval problem~\citep{akiyama2019first}.

\subsection{Recovery types}
One common characteristic of these problems is that information is lost and perfect recovery is impossible~\citep{tarantola2005inverse}, i.e. they are \textit{ill-posed}. Hence, the type of ``recovery'' we are looking for should be carefully defined~\citep{Scarlett_2022}. For instance, one might be looking for the point that maximizes the posterior distribution $p(\vx | \vy)$~\citep{bassett2019maximum, pereyra2019revisiting}. Often, the Maximum a posteriori (MAP) estimation coincides with the Minimum Mean Squared Error Estimator, i.e. the conditional expectation $\E[\vx | \vy]$~\citep{young2005essentials, murphy2012machine}.  MMSE estimation attempts to minimize distortion of the unknown signal $\vx$, but often lead to unrealistic recoveries. A different approach is to sample from the full posterior distribution, $p(\vx | \vy)$. Posterior sampling accounts for the uncertainty of the estimation, and typically produces samples that have higher perception quality. \citet{blau2018perception} show that, in general, it is impossible to find a sample that maximizes perception and minimizes distortion at the same time. Yet, posterior sampling is nearly optimal~\citep{mri_paper} in terms of distortion error.

\subsection{Approaches for Solving Inverse Problems}

Inverse problems have a rich history, with approaches evolving significantly over the decades~\citet{ribes2008linear,barrett2013foundations}. While a comprehensive review is beyond the scope of this survey, we highlight key trends to provide context.  Early approaches, prevalent in the 2000s, often framed inverse problems as optimization tasks~\citet{daubechies2004iterative,candes2006robust,donoho2006compressed, figueiredo2003algorithm,daubechies2004iterative,hale2007fixed, shlezinger2023model}. These methods sought to balance data fidelity with regularization terms that encouraged desired solution properties like smoothness~\citet{rudin1992nonlinear,beck2009fast} or sparsity in specific representations (e.g., wavelets, dictionaries)~\citet{figueiredo2003algorithm,daubechies2004iterative,candes2006robust,donoho2006compressed,hale2007fixed}.

The advent of deep learning brought a paradigm shift~\citet{ongie2020deep}.  Researchers began leveraging large paired datasets to directly learn mappings from measurements to clean signals using neural networks~\citet{dong2015image,lim2017enhanced,tao2018scale,chen2018learning,zamir2022restormer,chen2022simple,tu2022maxim,zamir2021multi}. These approaches focus on minimizing some reconstruction loss during training, with various techniques employed to penalize distortions, and optimize for specific application goals (e.g., perceptual quality~\citet{isola2017image,kupyn2018deblurgan}). Traditional point estimates aim to recover a single reconstruction by for example minimizing the average reconstruction error (i.e., MMSE) or by finding the most probable reconstruction through Maximum a Posteriori estimate (MAP), i.e., finding the $\vx$ that \textit{maximizes} $p(\vx | \vy)$. While powerful, this approach can suffer from ``regression to the mean'', where the network predicts an average solution that may lack important details or even be outside the desired solution space~\citet{blau2018perception,delbracio2023inversion}.  In fact, learning a mapping to minimize a certain distortion metric will lead, in the best case, to an average of all the plausible reconstructions (e.g., when using a L2 reconstruction loss, the best-case solution will be the posterior mean). This reconstruction might not be in the target space (e.g., a blurry image being the average of all plausible reconstructions)~\citet{blau2018perception}.

Recent research has revealed a striking connection between denoising algorithms and inverse problems. Powerful denoisers, often based on deep learning, implicitly encode valuable information about natural signals. By integrating these denoisers into optimization frameworks, we can harness their learned priors to achieve exceptional results in a variety of inverse problems~\citet{pnp,sreehari2016plug, chan2016plug,romano2016little,cohen2021regularization,kadkhodaie2021stochastic, kamilov2023plug,milanfar2024denoising}. This approach bridges the gap between traditional regularization methods and modern denoising techniques, offering a promising new paradigm for solving these challenging tasks.

An alternative perspective views inverse problems through the lens of Bayesian inference. Given measurements $\vy$, the goal becomes generating plausible reconstructions by sampling from the posterior distribution  $p(\vX|\vY=\vy)$ – the distribution of possible signals $\vx$ given the observed measurements $\vy$. 

In this survey we explore a specific class of methods that utilize diffusion models as priors for $p_\vX$, and then try to generate plausible reconstructions (e.g., by sampling from the posterior). While other approaches exist, such as directly learning conditional diffusion models or flows for specific inverse problems~\citet{li2021srdiff,saharia2021image,saharia2022palette,whang2022deblurring,luo2023image,luo2023refusion,albergo2023building,albergo2023stochastic,lipman2023flow,liu2023i2sb,liu2023flow,shi2023diffusion}, these often require retraining for each new application. In contrast, the methods covered in this survey offer a more general framework applicable to arbitrary inverse problems without retraining or fine-tuning.

\paragraph{Unsupervised methods.} We refer as \emph{unsupervised methods} to those that focus on characterizing the distribution of target signals, $p_\vX$, and applying this knowledge during the inversion process. Since they don't rely on paired data, they can be flexibly applied to different inverse problems using the same prior knowledge.

Unsupervised methods can be used to maximize the likelihood of $p(\vx | \vy)$ or to sample from this distribution. Algorithmically, to solve the former problem we typically use (some variation of) Gradient Descent and to solve the latter (some variation of) Monte Carlo Simulation (e.g., Langevin Dynamics). Either way, one typically requires to compute the gradient of the conditional log-likelihood, i.e., $\nabla_{\vx}\log p(\vx| \vy)$.

A simple application of Bayes Rule reveals that:
\begin{gather}
    \underbrace{\nabla_{\vx}\log p(\vx| \vy)}_{\textrm{conditional score}} = \underbrace{\nabla_{\vx} \log p(\vx)}_{\textrm{unconditional score}} + \underbrace{\nabla_{\vx} \log p(\vy|\vx)}_{\textrm{measurements matching term}}.
    \label{eq:conv_approach}
\end{gather}

The last term typically has a closed-form expression, e.g. for the linear case, we have that: $\nabla_{\vx}\log p(\vy|\vx) = \frac{\vy - A\vx}{\sigma_\vy^2}$. However, the first term, known as the score function, might be hard to estimate when the data lie on low-dimensional manifolds. The problem arises from the fact that we do not get observations outside of the manifold and hence the vector-field estimation is inaccurate in these regions.

One way to sidestep this issue is by using a ``smoothed'' version of the score function, representing the score function of noisy data that will be supported everywhere. The central idea behind diffusion generative models is to learn score functions that correspond to different levels of smoothing. Specifically, in diffusion modeling, we attempt to learn the smoothed score functions, $\nabla_{\vx_t} \log p_t(\vx_t)$, where $\vX_t = \vX_0 + \sigma_t \vZ, \quad \vZ \sim \mathcal N(\bm{0}, I)$, for different noise levels $t$. During sampling, we progressively move from more smoothed vector fields to the true score function. At the very end, the score function corresponding to the data distribution is only queried at points for which the estimation is accurate because of the warm-start effect of the sampling method.

Even though estimating the unconditional score becomes easier (because of the smoothing), the measurement matching term becomes \textit{time dependent} and loses its closed form expression. Indeed, the likelihood of the measurements is given by the intractable integral:
\begin{gather}
\label{eq:intractable_integral}
    p_t(\vy|\vx_t) = \int_{} p(\vy|\vx_0) p(\vx_0|\vx_t)\d\vx_0.
\end{gather}

The computational challenge that emerges from the intractability of the conditional likelihood has led to the proposal of numerous approaches to use diffusion models to solve inverse problems~\citep{mri_paper, dps, song2022pseudoinverse, ncsnv3, ilvr, ddrm, kawar2021snips, ddnm, dds, diffpir, rozet2024learning, lugmayr2022repaint, rout2023theoretical, mardani2024variational, feng2023score, feng2023efficient, wu2024principled, dou2023diffusion, cardoso2023monte, wu2023practical, kadkhodaie2020solving, wang2024dmplug, chihaoui2024zeroshot, chung2024deep, shen2024understanding}. The sheer number of the proposed methods, but also the different perspectives under which these methods have been developed, make it hard for both newcomers and experts in the field to understand the connections between them and the unifying underlying principles. This work attempts to \textit{explain}, \textit{taxonomize} and \textit{relate} prominent methods in the field of using diffusion models for inverse problems. Our list of methods is by no means exhaustive. The goal of this manuscript is not to list all the methods that have been proposed but to review some representative methods of different approaches and present them under a unifying framework. We believe this survey will be useful as a reference point for people interested in this field.

\section{Background}
\subsection{Diffusion Processes}
\paragraph{Forward and Reverse Processes.}
The idea of a diffusion model is to transform a a simple distribution (e.g., normal distribution) into the unknown data distribution $p_0(\vx)$, that we don't know explicitly but we have access to some of its samples. The first step is to define a \textit{corruption process}. The popular Denoising Diffusion Probabilistic Models (DDPM)~\citet{ddpm, ncsn}, adopt a discrete time Markovian process to transform the input Normal distribution into the target one by incrementally adding Gaussian noise. More generally, the corruption processes of interest can be generalized to continuous time by a stochastic differential eqaution (SDE)~\citep{ncsnv3}:
\begin{gather}
    \dx_t = \underbrace{\vf(\vx_t, t)}_{\mathrm{drift \ coeff.}}\dt + \underbrace{g(t)}_{\mathrm{diffusion \ coeff.}}\dW_t,
    \label{eq:forward_process}
\end{gather}
with $\vx_0 \sim p_0, \vx_0 \in \R^n$, and $\bm{W}_t$ denotes a Wiener process (i.e., Brownian motion). 
This SDE gradually transforms the data distribution into Gaussian noise. We denote with $p_t$ the distribution that arises by running this dynamical system up to time $t$.

A remarkable result by \citet{anderson} shows that we can sample from $p_0$ by running backwards in time the reverse SDE:
\begin{gather}
    \dx_t = \left(\vf(\vx_t, t) - g^2(t)\underbrace{\nabla_{\vx_t}\log p_t(\vx_t)}_{\mathrm{score}}\right)\dt + g(t)\dW_t,
    \label{eq:backward_process}
\end{gather}
initialized at $\vx_T \sim p_T$. For sufficiently large $T$ and for linear drift functions $\vf(\cdot, \cdot)$, the latter distribution approaches a Gaussian distribution with known parameters that can be used for initializing the process. Hence, the remaining goal becomes to estimate the score function $\nabla_{\vx_t} \log p_t(\vx_t)$.

\paragraph{Probability Flow ODE.} \citet{ncsnv3, maoutsa2020interacting} observe that the (deterministic) differential equation:
\begin{gather}
    \frac{\dx_t}{\dt} = \left(\vf(\vx_t, t) - \frac{g^2(t)}{2}\underbrace{\nabla_{\vx_t}\log p_t(\vx_t)}_{\mathrm{score}}\right)
    \label{eq:det_backward_process}
\end{gather}
corresponds to the same Fokker-Planck equations as the SDE of \Eqref{eq:backward_process}. An implication of this is we can use the \textit{deterministic} sampling scheme of \Eqref{eq:det_backward_process}. Any well-built numerical ODE solver can be used to solve \Eqref{eq:det_backward_process}, such as the Euler solver:
\begin{align}
    \vx_{t-\Delta t} = \vx_t + \Delta t \left(\vf(\vx_t, t) - \frac{g^2(t)}{2}\nabla_{\vx_t}\log p_t(\vx_t)\right).
    \label{eq:euler_discretized_reverse_ode}
\end{align}

\paragraph{SDE variants: Variance Exploding and Variance Preserving Processes.} The drift coefficients, $\vf(\vx_t, t)$, and the diffusion coefficients $g(t)$ are design choices. One popular choice, known as the Variance Exploding SDE, is setting $\vf(\vx_t, t) = \bm{0}$ and $g(t) = \sqrt{\frac{\mathrm{d}\sigma_t^2}{\mathrm{d}t}}$ for some variance scheduling $\{\sigma_t\}_{t=0}^{T}$. Under these choices, the marginal distribution at time $t$ of the forward process of \Eqref{eq:forward_process} can be alternatively described as:
\begin{gather}
    \vX_t = \vX_0 + \sigma_t \vZ, \quad \vX_0 \sim p(\vX_0), \quad \vZ \sim \mathcal N(\bm{0}, I_n).
    \label{eq:ve_marginals}
\end{gather}
The typical noise scheduling for this SDE is $\sigma_t = \sqrt{t}$ (that corresponds to $g(t)=1$).

Another popular choice is to set the drift function to be $\vf(\vx_t, t) = - \vx_t$, which is known as the Variance Preserving (VP) SDE. A famous process in the VP SDE family is the Ornstein–Uhlenbeck (OU) process:
\begin{gather}
    \dx_t = -\vx_t \dt + \sqrt{2}\dW_t,
\end{gather}
which gives:
\begin{gather}
 \vX_t = \exp(-t) \vX_0 + \sqrt{1 - \exp(-2t)}\vZ, \quad \vZ \sim \mathcal N(\bm{0}, I_n).
\end{gather}

The VP SDE~\citep{ddpm} takes a more general form: 
\begin{gather}
    \vX_t = \sqrt{\alpha_t} \vX_0 + (1-\alpha_t) \vZ, \quad \vX_0 \sim p(\vX_0), \quad \vZ \sim \mathcal N(\bm{0}, I_n).
    \label{eq:vp_marginals}
\end{gather}
With reparametrization and the Euler solver, this leads to an efficient solution to \Eqref{eq:det_backward_process}, known as DDIM~\citep{ddim}: 
\begin{align}\label{eq:ddim}
    \vx_{t-1} = \sqrt{\alpha_{t-1}}\underbrace{\Bigg(\frac{\vx_t+(1-\alpha_t)\nabla_{\vx_t}\log p_t(\vx_t)}{\sqrt{\alpha_t}}\Bigg)}_{=:\widehat{\vx}_0=\text{predicted }\vx_0}+\sqrt{1-\alpha_{t-1}-\sigma_t^2}\underbrace{\Bigg(-\sqrt{1-\alpha_t} \nabla_{\vx_t}\log p_t(\vx_t) \Bigg)}_{\text{direction toward }\vx_t}.
\end{align}
For convenience, in the rest of the paper, this update will be written as: $\vx_{t-1}\leftarrow \mathrm{UnconditionalDDIM}(\widehat{\vx}_0, \vx_t)$.

\subsection{Tweedie's Formula and Denoising Score Matching} In what follows, we will discuss how one can learn the score function $\nabla_{\vx_t}\log p_t(\vx_t)$ that appears in \Eqref{eq:bayes_rule}. We will focus on the VE SDE, since the mathematical calculations are simpler. 

Tweedie's formula~\citep{efron2011tweedie} is a famous result in statistics that shows that for an additive Gaussian corruption, $\vX_t = \vX_0 + \sigma_t \vZ, \vZ \sim \mathcal N(\mathbf{0}, I_n)$, it holds that:
\begin{gather}
    \nabla_{\vx_t} \log p_t(\vx_t) = \frac{\E[\vX_0 | \vX_t=\vx_t] - \vx_t}{\sigma_t^2}.
    \label{eq:tweedies}
\end{gather}

The formal statement and a self-contained proof can be found in the Appendix, Lemma \ref{lemma:Tweedies}.

Tweedie's formula gives us a way to derive the unconditional score function needed in \Eqref{eq:bayes_rule}, by optimizing for the conditional expectation, $\E[\vX_0 |  \vX_t=\vx_t]$. The conditional expectation $\E[\vX_0 |  \vX_t=\vx_t]$, is nothing more than the minimum mean square error estimator (MMSE) of the clean image given the noisy observation $\vx_t$, that is a \emph{denoiser}.

In practice, we don't know analytically this denoiser but we can parametrize it using a neural network $\vh_{\bm{\theta}}(\vx_t)$ and learn it in a supervised way by minimizing the following objective:
\begin{gather}
    J_{\mathrm{DSM}}({\bm{\theta}}) = \E_{\vx_0, \vx_t}\left[ \norm{\net{\vx_t} - \vx_0} \right].
    \label{eq:dsm}
\end{gather}
Assuming a rich enough family $\bm{\Theta}$, the minimizer of \Eqref{eq:dsm} is $\vh_{\theta}(\vx_t) = \E[\vx_0 | \vX_t=\vx_t]$ (see Lemma \ref{lemma:mmse_exp}) and the score in \Eqref{eq:tweedies} is approximated as $\big(\vh_\theta(\vx_t) - \vx_t\big)/\sigma_t^2$. 
Note that for each $\sigma_t$ we would need to learn a different denoiser (since the noise strength is different), or alternative the neural network $\vh_{{\bm{\theta}}}$ should also take as input the value of $t$ or $\sigma_t$. Diffusion models are trained following the later paradigm, i.e. the same neural network approximates the optimal denoisers at all noise levels by conditioning it on the noise level through $t$.

Interestingly, \citet{vincent2011connection} independently discovered that the score function can be learned by minimizing an $l_2$ objective, similar to \Eqref{eq:dsm}. The formal statement and a self-contained proof of this alternative derivation is included in the Appendix, Theorem \ref{theorem:dsm}.

\subsection{Latent Diffusion Processes}
\label{sec:ldms}
For high-dimensional distributions, diffusion models training (see Equation \ref{eq:dsm}) and sampling (see Equation \ref{eq:det_backward_process}) require massive computational resources. To make the training and sampling more efficient, the authors of Stable Diffusion \citet{ldm} propose performing the diffusion in the latent space of a pre-trained powerful autoencoder. Specifically, given an encoder $\mathrm{Enc}: \R^n \to \R^k$ and a decoder $\mathrm{Dec}: \R^k \to \R^n$, one can create noisy samples:
\begin{gather}
     \vX^\mathrm{E}_t = \underbrace{\vX_0^\mathrm{E}}_{\mathrm{Enc}(\vX_0)} + \sigma_t \vZ, \quad \vZ \sim \mathcal N(\bm 0, I_k),
    \label{eq:latent_noisy_iterates}
\end{gather}
and train a denoiser network in the latent space. At inference time, one starts with pure noise, samples a clean latent $\tilde \vx^\mathrm{E}_0$ by running the reverse process,  and outputs $\vx_0 = \mathrm{Dec}(\tilde \vx^\mathrm{E}_0)$\footnote{To give an idea, the original work introducing latent diffusion models~\citet{ldm}, propose to use a latent space of dimension $64\times 64 \times 4$ to represent images of size $512 \times 512 \times 3$.}. Solving inverse problems with Latent Diffusion models requires special treatment. We discuss the reasons and approaches in this space in Section \ref{sec:solving_inv_with_ldms}.

\subsection{Conditional Sampling} 
\subsubsection{Stochastic Samplers for Inverse Problems}
The goal in inverse problems is to sample from $p_0(\cdot | \vy)$ assuming a corruption model $\vY = \mathcal A(\vX_0) + \sigma_\vy \vZ, \ \vZ \sim \mathcal N(\mathbf{0}, I_m)$. We can easily adapt the original unconditional formulation given by \Eqref{eq:backward_process} into a conditional one to generate samples from $p_0(\cdot | \vy)$. Specifically, the associated reverse process is given by the stochastic dynamical system~\citep{oksendal2010stochastic}:
\begin{gather}
    \dx_t = \left(\vf(\vx_t, t) - g^2(t)\underbrace{\nabla_{\vx_t}\log p_t(\vx_t| \ \vy)}_{\mathrm{conditional\, score}}\right)\dt + g(t)\dW_t,
    \label{eq:cond_backward_process}
\end{gather}
initialized at $\vx_T \sim p_T(\cdot | \vy)$. For sufficiently large $T$ and for linear drift functions $\vf(\cdot, \cdot)$, the distribution $p_T(\cdot | \vy)$ is a Gaussian distribution with parameters independent of $\vy$. In the conditional case, the goal becomes to estimate the score function $\nabla_{\vx_t} \log p_t(\vx_t | \vy)$.

\subsubsection{Deterministic Samplers for Inverse Problems} It is worth noting that (as in the unconditional setting) it is possible to derive deterministic sampling algorithms as well. Particularly, one can use the following dynamical system~\citep{ncsnv3, oksendal2010stochastic}:
\begin{gather}
        \frac{\dx_t}{\dt} = - \frac{g^2(t)}{2}\nabla_{\vx_t}\log p_t(\vx_t|\vy).
\end{gather}
initialized at $p_T(\cdot | \vy)$ to get sample from the conditional distribution $p_0(\cdot | \vy)$. 
Once again, to run this discrete dynamical system, one needs to know the conditional score, $\nabla_{\vx_t} \log p_t(\vx_t|\vy)$.

\subsubsection{Conditional Diffusion Models}
Similarly to the unconditional setting, one can directly train a network to approximate the conditional score, $\nabla_{\vx_t} \log p_t(\vx_t|\vy)$. A generalization of Tweedie's formula gives that:
\begin{gather}
    \nabla \log p_t(\vx_t | \vy) = \frac{\E[\vX_0 | \vX_t=\vx_t, \vY=\vy] - \vx_t}{\sigma_t^2}.
    \label{eq:cond_tweedie}
\end{gather}
Hence, one can train a network using a generalized version of the Denoising Score Matching, 
\begin{gather}
    J_{\mathrm{cond, \ DSM}}({\bm{\theta}}) = \E_{\vx_0, \vx_t, \vy}\left[ \norm{\net{\vx_t, \vy} - \vx_0} \right],
    \label{eq:cond_dsm}
\end{gather}
and then use it in \Eqref{eq:cond_tweedie} in place of the conditional expectation. The main issue with this approach is that the forward model (degradation operator) needs to be known at training time. If the corruption model $\mathcal A(\vX)$ changes, then the model needs to be retrained. Further, with this approach we need to train new models and we cannot directly leverage powerful unconditional models that are already available. The focus of this work is on methods that use pre-trained unconditional diffusion models to solve inverse problems, without further training.   

\subsubsection{Using pre-trained diffusion models to solve inverse problems}
\label{sec:uncond_for_inverse}
As we showed earlier, the conditional score can be decomposed using Bayes Rule into:
\begin{gather}
    \nabla_{\vx_t} \log p_t(\vx_t | \vy) = \underbrace{\nabla_{\vx_t} \log p_t(\vx_t)}_{\mathrm{score}} + \underbrace{\nabla_{\vx_t}\log p(\vy | \vx_t)}_{\mathrm{measurements \ matching \ term}}.
    \label{eq:bayes_rule}
\end{gather}
that is, the (smoothed) score function, and the measurements matching term that is given by the inverse problem we are interested in solving. Applying this to \eqref{eq:cond_backward_process}, we get that:
\begin{gather}
    \dx_t = \left(\vf(\vx_t, t) - g^2(t)\left(\nabla_{\vx_t}\log p_t(\vx_t) + \nabla \log p(\vy | \vx_t)\right)\right)\dt + g(t)\dW_t.
    \label{eq:backwards_sde_after_bayes}
\end{gather}
Similarly, one can use the deterministic process:
\begin{gather}
    \dx_t = \left(\vf(\vx_t, t) - \frac{1}{2}g^2(t)\left(\nabla_{\vx_t}\log p_t(\vx_t) + \nabla \log p(\vy | \vx_t)\right)\right)\dt.
    \label{eq:backwards_ode_after_bayes}
\end{gather}

We have already discussed how to train a neural network to approximate $\nabla_{\vx_t}\log p_t(\vx_t)$ (using Tweedie's Formula / Denoising Score Matching). However, here we further need access to the term $\nabla_{\vx_t}\log p(\vy | \vx_t)$. The likelihood of the measurements is given by the intractable integral:
\begin{gather}
    p_t(\vy|\vx_t) = \int_{} p(\vy|\vx_0) p(\vx_0|\vx_t)\d\vx_0.
\end{gather}

Gupta et. al~\citep{gupta2024diffusion} prove that there are instances of the posterior sampling problem for which \emph{every} algorithm takes superpolynomial time, even though \emph{unconditional} sampling is provably fast. Hence, diffusion models excel at performing unconditional sampling but are hard to use as priors for solving inverse problems because of the time dependence in the measurements matching term. Since the very introduction of diffusion models, there has been a plethora of methods proposed to use them to solve inverse problems without retraining. This survey serves as a reference point for different techniques that have been developed in this space.

\subsubsection{Ambient Diffusion: Learning to solve inverse problems using only measurements} 
The goal of the unsupervised learning approach for solving inverse problems (Section \ref{sec:uncond_for_inverse}) is to use a prior $p(\vx)$ to approximate the measurements matching term, $\nabla \log p_t(\vy|\vx_t)$. However, in certain applications, it is expensive or even impossible to get data from (and hence learn) $p(\vx)$ in the first place. For instance, in MRI the quality of the data is proportionate to the time spent under the scanner~\citep{zbontar2018fastmri} and it is infeasible to acquire full measurements from black holes~\citep {akiyama2019first}. This creates a chicken-egg problem: we need access to $p(\vx)$ to solve inverse problems and we do not have access to samples from $p(\vx)$ unless we can solve inverse problems. In certain scenarios, it is possible to break this seemingly impossible cycle. 

Ambient Diffusion~\citet{ambientdiffusion} was one of the first frameworks to train diffusion models with linearly corrupted data. The key concept behind the Ambient Diffusion framework is the idea of further corruption. Specifically, the given measurements get further corrupted and the model is trained to predict a clean image by using the measurements before further corruption for validation. Ambient DPS~\citep{aali2024ambient} shows that priors learned from corrupted data can even outperform (in terms of usefulness for inverse problems), at the high-corruption regime, priors learned from clean data. Ambient Diffusion was extended to handle additive Gaussian Noise in the measurements. The paper Consistent Diffusion Meets Tweedie~\citet{constweedie} was the first diffusion-based framework to provide guarantees for sampling from the distribution of interest, given only access to noisy data. This paper extends the idea of further corruption to the noisy case and proposes a novel consistency loss~\citet{daras2023consistent} to learn the score function for diffusion times that correspond to noise levels below the level of the noise in the dataset. 

Both Ambient Diffusion and Consistent Diffusion Meets Tweedie have connections to deep ideas from the literature in learning restoration models from corrupted data, such as Stein’s Unbiased Risk Estimate (SURE)~\citet{SURE, stein1981estimation} and Noise2X~\citet{noise2noise,krull2019noise2void,batson2019noise2self}. These connections are also leveraged by alternative frameworks to Ambient Diffusion, as in~\citep{kawar2021snips, aali2023solving}. A different approach for learning diffusion models from measurements is based on the Expectation-Maximization (EM) algorithm~\citep{bai2024expectation, rozet2024learning, wang2024integrating}. The convergence of these methods to the true distribution depends on the convergence of the EM algorithm, which might get stuck in a local minimum.

In this survey, we focus on the setting where a pre-trained prior $p(\vx)$ is available, regardless of whether it was learned from clean or corrupted data.

\section{Reconstruction Algorithms}
We summarize all the methods analyzed in this work in Table \ref{tab:categorization}. The methods have been taxonomized based on the approach they use to solve the inverse problem (explicit score approximations, variational methods, CSGM-type methods and asymptotically exact methods), the type of inverse problems they can solve and the optimization techniques used to solve the problem at hand (gradient descent, sampling, projections, parameter optimization). Additionally, we provide links to the official code repositories associated with the papers included in this survey.  Please note that we have not conducted a review or evaluation of these codebases to verify their consistency with the corresponding papers. These links are included for informational purposes only.

\paragraph{Taxonomy based on the type of the reconstruction algorithm.}
We identified four families of methods. \textit{Explicit Approximations for Measurement Matching:} These methods approximate the measurement matching score, $\nabla \log p_t(\vy|\vx_t)$, with a closed-form expression. \textit{Variational Inference:} These methods approximate the true posterior distribution, $p(\vx | \vy)$, with a simpler, tractable distribution. Variational formulations are then used to optimize the parameters of this simpler distribution. \textit{CSGM-type methods:} The works in this category use backpropagation to change the initial noise of the deterministic diffusion sampler, essentially optimizing over a latent space for the diffusion model.
\textit{Asymptotically Exact Methods:} These methods aim to sample from the true posterior distribution. This is typically achieved by constructing Markov chains (MCMC) or by propagating particles through a sequence of distributions (SMC) to obtain samples that approximate the posterior. Methods that do not fall into any of these categories are classified as \textit{Others}.

\paragraph{Taxonomy based on the type of optimization techniques used.} The objective of all methods is to explain the measurements. The measurement consistency can be enforced with different optimization techniques, e.g. through gradients (Grad), projections (Proj), sampling (Samp), or other optimization techniques (Opt). Methods that belong to the Grad-type take a single gradient step (either it be deterministic or stochastic) to $\vx_t$ to enforce measurement consistency. Proj-type projects $\vx_t$ or $\E[\vX_0|\vX_t=\vx_t]$ to the measurement subspace. Samp-type samples the next particles by defining a proposal distribution, and propagates multiple chains of particles to solve the problem. Opt-type either defines and solves an optimization problem for every timestep, or defines a global optimization problem that encompasses all timesteps. When the method belongs to more than one type, we seperate them with /. Note that the categorization of different ``types'' is subjective, and more often than not, the category that the method belongs to may be interpreted in multiple ways. For instance, a projection step is also a gradient descent step with a specific step size.

\paragraph{Taxonomy based on the type of the inverse problem.}

Based on the linearity of the corruption operator $\mathcal{A}$, the inverse problems can be classified as linear or nonlinear. The inverse problems can be further categorized based on whether there is noise in the measurements. Additionally, they are classified as non-blind or blind depending on whether full information about $\mathcal{A}$ is available. In blind problems, the degradation operator (e.g., convolution kernel, inpainting kernel) is known, while its coefficients are unknown but parametrized. For example, we might know that we have measurements with additive Gaussian noise, but the variance of the noise might be unknown. Finally, in certain inverse problems, there is additional text-conditioning. Such inverse problems are typically solved with text-to-image latent diffusion models~\citep{ldm}.

\subsection{Explicit Approximations for the Measurements Matching Term}

The first family of reconstruction algorithms we identify is the one were explicit approximations for the measurements matching term, $\nabla_{\vx_t}  \log p(\vy | \vX_t=\vx_t)$, are made. It is important to underline that these approximations are not always clearly stated in the works that propose them, 
which makes it hard to understand the differences and commonalities between different methods. In what follows, we attempt to elucidate the different approximations that are being made and present different works under a common framework. To provide some insights, we often provide the explicit approximation formulas for the measurements matching term in the setting of linear inverse problems. In general, it follows the template form:

\begin{equation}
    \nabla_{\vx_t} \log p(\vy | \vX_t=\vx_t) \approx - \frac{\eqnmarkbox[Plum]{proj}{\mathcal{L}_t}\eqnmarkbox[NavyBlue]{merror}{\,\,\,\,\mathcal{M}_t\,\,\,\,}}{\eqnmarkbox[OliveGreen]{guidance}{\,\,\mathcal{G}_t\,\,}}.
\end{equation}

Here,
\begin{itemize}
    \item $\mathcal{M}_t$ represents the error vector measuring the discrepancy between the observation $\vy$ and the estimated restored vector; for example, in Score ALD~\citep{mri_paper}, $\mathcal{M}_t=\vy - A\vx_t$.
    \item $\mathcal{L}_t$ denotes a matrix that projects the error vector  $\mathcal{M}_t$ from $\mathbb{R}^m$ back into an appropriate space in $\mathbb{R}^n$; for instance, in Score ALD,  $\mathcal{L}_t= A^\top$.
    \item $\mathcal{G}_t$ is the re-scaling scalar for the guidance vector $\mathcal{L}_t\mathcal{M}_t$; for example, in Score ALD, $\mathcal{G}_t=\sigma_\vy^2 + \gamma_t^2$ with a hyperparameter $\gamma_t$.
\end{itemize}
In Figure~\ref{fig:tree}, we summarize the approximation-based methods in this section using the template above. We use $\propto$ to omit the guidance strength terms $\mathcal{G}_t$.

\setcounter{subsubsection}{-1}
\subsubsection{Sampling from a Denoiser~\citet{kadkhodaie2020solving}}

\citet{kadkhodaie2020solving} introduce a method for solving linear inverse problems by using the implicit prior knowledge captured by a pre-trained denoiser on multiple noise levels. The method is anchored on Tweedie's formula that connects the least-squares solution for Gaussian denoising to the gradient of the log-density of noisy images given in \Eqref{eq:tweedies}
\begin{align}
\hat{\vx}(\vy) = \vy + \sigma^2 \nabla_\vy \log p(\vy),
\end{align}
where $\vy=\vx + \vn$, $\vn \sim \mathcal{N}(\mathbf{0}, \sigma^2 I_n)$.

By interpreting the denoiser's output as an approximation of this gradient, the authors develop a stochastic gradient ascent algorithm to generate high-probability samples from the implicit prior
\begin{align}
\vy_t = \vy_{t-1} + h_t \vr(\vy_{t-1}) + \epsilon_t \vz_t,
\end{align}
where $\vr(\vy) = \hat{\vx}(\vy) - \vy$ is the denoiser residual, $h_t$ is a step size (parameter), and $\epsilon_t$ controls the amount of newly introduced Gaussian noise $\vz_t$.

To solve linear inverse problems such as deblurring, super-resolution, and compressive sensing, the generative method is extended to handle constrained sampling. Given a set of linear measurements $\vx_c = M^{\top} \vx$ of an image $\vx$, where $M$ is a low-rank measurement matrix, the goal is to reconstruct the original image by utilizing the following gradient:
\begin{align}
\nabla_\vy \log p(\vy|\vx_c) = (I - MM^{\top})\vr(\vy) + M(\vx_c - M^{\top}\vy),
\end{align}

This approach is particularly interesting because its mathematical foundation relies solely on Tweedie's formula, providing a simple yet powerful framework for tackling inverse problems using denoisers.

\subsubsection{Score ALD~\citep{mri_paper}}
One of the first proposed methods for solving \textit{linear} inverse problems with diffusion models is the 
Score-Based Annealed Langevin Dynamics (Score ALD)~\citep{mri_paper} method. The approximation of this work is that:

\begin{equation}
    \nabla_{\vx_t} \log p(\vy | \vX_t=\vx_t) \approx - \frac{\eqnmarkbox[Plum]{proj}{A^{\top}}\eqnmarkbox[NavyBlue]{merror}{\left( \vy - A\vx_t\right)}}{\eqnmarkbox[OliveGreen]{guidance}{\sigma_{\vy}^2 + \gamma_t^2}},
    \label{eq:score_ALD}
\end{equation}
where $\gamma_t$ is a parameter to be tuned. 
\annotate[yshift=1em]{above}{merror}{measurements error}
\annotate[yshift=1em]{above,left}{proj}{``lifting'' matrix}
\annotate[yshift=-0.5em,xshift=3em]{below,right}{guidance}{guidance strength}

It is pretty straightforward to understand what this term is doing. The diffusion process is guided towards the opposite direction of the ``lifting'' (application of the $A^{\top}$ operator) of the measurements error, i.e. $(\vy - A\vx_t)$, where the denominator controls the guidance strength.

\subsubsection{Score-SDE~\citep{ncsnv3}}
Score-SDE~\citep{ncsnv3} is another one of the first works that discussed solving inverse problems with pre-trained diffusion models. For linear inverse problems, the difference between Score-ALD and Score-SDE is that the latter noises the measurements before computing the measurements error. Specifically, for $t: \sigma_t > \sigma_\vy$, the approximation becomes:
\vspace{1.5em}
\begin{gather}
\label{eq:score-sde-approx}
    \nabla_{\vx_t} \log p(\vy|\vX_t = \vx_t) \approx - \eqnmarkbox[Plum]{proj}{A^\top} (\eqnmarkbox[RoyalPurple]{noisymeasurements}{\vy + \sigma_t\bm{\epsilon}} - A\vx_t)
\end{gather}
\annotate[yshift=0.5em]{above,left}{proj}{``lifting'' matrix}
\annotate[yshift=0.5em]{above}{noisymeasurements}{$y_t$ (noised measurements)}

where $\bm{\epsilon}$ is sampled from $\mathcal N(\bm{0}, I_m)$. Here, $A$ is an orthogonal matrix, and taking a gradient step with \Eqref{eq:score-sde-approx} yields a noisy projection to $\vy_t = A\vx_t$ where $\vy_t = \vy + \sigma_t\bm{\epsilon}$. Hence, we categorize Score-SDE as ``projection''.

Disregarding the guidance strength of \Eqref{eq:score_ALD}, \Eqref{eq:score_ALD} and \Eqref{eq:score-sde-approx} look very similar. Indeed, the only difference is that the latter has stochasticity that arises from the noising of the measurements.

\paragraph{Special case: Inpainting (Repaint~\citep{lugmayr2022repaint})}
Observe that for the simplest case of inpainting, \Eqref{eq:score-sde-approx} would be replacing the pixel values in the current estimate $\vx_t$ with the known pixel values from the {\em noised} $\vy_t$. Coincidentally, this is exactly the Repaint~\citet{lugmayr2022repaint} algorithm that was proposed for solving the inpainting inverse problem with pre-trained diffusion models. RePaint++~\citet{rout2023theoretical} improves upon this approximation to run the forward-reverse diffusion processes multiple times, so that the errors arising (e.g. boundaries) can be mitigated. This can be thought of as analogous to running MCMC corrector steps as in predictor-corrector sampling~\citep{ncsnv3}.

\subsubsection{ILVR~\citep{ilvr}}
ILVR is a similar approach that was initially proposed for the task of super-resolution. The approximation made here is the following:
\vspace{1em}
\begin{align}
    \nabla_{\vx_t} \log p(\vy|\vX_t = \vx_t) \approx -A^\dagger(\vy_t - A\vx_t) = -\eqnmarkbox[Plum]{moore}{(A^\top A)^{-1} A^\top} (\eqnmarkbox[NavyBlue]{merror}{\vy_t - A\vx_t)},
\end{align}
where $A^\dagger$ is the Moore-Penrose pseudo-inverse of $A$, and similar to Score-SDE, $\vy_t = \vy + \sigma_t \bm{\epsilon}$. 
\annotate[yshift=0.5em]{above,left}{moore}{``lifting'' matrix}
\annotate[yshift=0.5em]{above}{merror}{measurements error}

ILVR can be regarded as a pre-conditioned version of score-SDE. In ILVR, the projection to the space of images happens using the Moore-Penrose pseudo-inverse of A, instead of the simple $A^{\top}$.

\subsubsection{DPS}
All of the previous algorithms were proposed for linear inverse problems. Diffusion Posterior Sampling (DPS) is one of the most well known reconstruction algorithms for solving non-linear inverse problems. The underlying approximation behind DPS is that:

\begin{gather}
    \nabla_{\vx_t} \log p(\vy | \vX_t=\vx_t) \approx \nabla_{\vx_t} \log p(\vy | \vX_0 = \E[\vX_0 | \vX_t=\vx_t]).
    \label{eq:DPS_approx}
\end{gather}

It is easy to see that:
\begin{gather}
    p(\vy | \vX_0 = \E[\vX_0 | \vX_t=\vx_t]) = \mathcal N\left(\vy; \vmu = \mathcal A(\E[\vX_0 | \vX_t=\vx_t]), \Sigma=\sigma_{\vy}^2I\right).
\end{gather}
Hence, the DPS approximation can be stated as:

\begin{gather}
    \nabla_{\vx_t} \log p(\vy | \vX_t=\vx_t) \approx \nabla_{\vx_t} \log \mathcal N\left(\vy; \vmu = \mathcal A(\E[\vX_0 | \vX_t=\vx_t]), \Sigma=\sigma_{\vy}^2I\right) \\
    = \nabla_{\vx_t} \left( \frac{1}{2\sigma_{\vy}^2}||\vy - \mathcal A(\E[\vX_0 | \vX_t=\vx_t])||^2\right) \\
    = \frac{1}{2\sigma_{\vy}^2}  \nabla_{\vx_t}^{\top}\mathcal A(\E[\vX_0 | \vX_t=\vx_t])\left(\mathcal A(\E[\vX_0 | \vX_t=\vx_t]) - \vy\right).
\end{gather}
For linear inverse problems, this simplifies to: \vspace{1em}
\begin{gather}
    \nabla_{\vx_t} \log p(\vy | \vX_t=\vx_t) \approx -\eqnmarkbox[OliveGreen]{guidance}{\frac{1}{2\sigma_{\vy}^2}}  \eqnmarkbox[Plum]{proj}{\nabla_{\vx_t}^{\top}\mathcal \E[\vX_0 | \vX_t=\vx_t]A^{\top}}\left(\eqnmarkbox[NavyBlue]{merror}{\vy - A\E[\vX_0 | \vX_t=\vx_t]}\right).
    \label{eq:linear_dps}
\end{gather}
\annotate[yshift=0.5em]{above}{merror}{measurements error}
\annotate[yshift=0.5em]{above,left}{proj}{``lifting'' matrix}
\annotate[yshift=-0.25em,xshift=-1em]{below,left}{guidance}{guidance strength}
\vspace{1em}

We can further use Tweedie's formula to further write it as:
\begin{gather}
   \nabla_{\vx_t} \log p(\vy | \vX_t=\vx_t) \approx -\eqnmarkbox[OliveGreen]{guidance}{\frac{1}{2\sigma_{\vy}^2}}\eqnmarkbox[Plum]{proj}{\left(I + \nabla_{\vx_t}^2\log p_{t}(\vx_t) \right)^{\top}A^{\top}}\left(\eqnmarkbox[NavyBlue]{merror}{\vy - A\E[\vX_0 | \vX_t=\vx_t]}\right).
   \label{eq:dps_written_for_score}
\end{gather}
In practice, DPS does not use the theoretical guidance strength but instead proposes to use a reweighting with a step size inversely proportional to the norm of the measurement error.

MCG~\citet{mcg} provides a geometric interpretation of DPS by showing that the approximation used in DPS can guarantee the noisy samples stay on the manifold. DSG~\citet{yang2024guidance} showed that one can choose a theoretically ``correct'' step size under the geometric view of MCG, and combined with projected gradient descent, one can achieve superior sample quality. MPGD~\citet{he2023manifold} showed that by constraining the gradient update step to stay on the low dimensional subspace by autoencoding, one can acquire better results.

\subsubsection{$\Pi$GDM~\citet{song2022pseudoinverse}}

Recall the intractable integral in \Eqref{eq:intractable_integral}.
According to this relation, the DPS approximation is achieved by setting
\begin{align}
    p(\vX_0|\vX_t) \approx \delta(\vX_0 - \E[\vX_0|\vX_t = \vx_t]).
\end{align}
In $\Pi$GDM, the authors propose to use a Gaussian distribution for approximation
\begin{align}
    p(\vX_0|\vX_t) \approx \mathcal{N}(\E[\vX_0|\vX_t = \vx_t], r_t^2 I_n),
\end{align}
where $r_t$ is a hyperparameter. For linear inverse problems, this leads to
\begin{align}
    p(\vy|\vX_t) \approx  \mathcal{N}(A\hat\E[\vX_0|\vX_t = \vx_t], r_t^2 AA^\top + \sigma_{\vy}^2 I_n).
\end{align}
Subsequently, we have
\begin{gather}
    \nabla_{\vx_t} \log p(\vy | \vX_t=\vx_t) \nonumber \\
    \approx -\eqnmarkbox[Plum]{proj}{\frac{\partial \E[\vX_0|\vX_t = \vx_t]}{\partial \vx_t}  (r_t^2 AA^\top + \sigma_{\vy}^2 I)^{-1} A^\top} (\eqnmarkbox[NavyBlue]{merror}{\vy - A \E[\vX_0|\vX_t = \vx_t]}).
\label{eq:pigdm}
\end{gather}
\annotate[yshift=-0.5em]{below,left}{proj}{``lifting'' matrix}
\annotate[yshift=-0.5em]{below}{merror}{measurements error}
\vspace{1em}

\subsubsection{Moment Matching~\citep{rozet2024learning}}
In $\Pi$GDM, the distribution $p(\vx_0 | \vx_t)$ was assumed to be isotropic Gaussian. However, one can calculate explicitly the variance matrix, $V[\vx_0 | \vx_t]$. As shown in Lemma \ref{lemma:score_jacobian}, it holds that:
\begin{gather}
    V[\vx_0 | \vx_t] = \sigma_t^4 H(\log p_t(\vx_t)) + \sigma_t^2I_n \\
    = \sigma_t^2 \nabla_{\vx_t}\E[\vx_0 | \vx_t].
\end{gather}
The Moment Matching~\citep{rozet2024learning} method approximates the distribution $p(\vx_0 | \vx_t)$ with an anisotropic Gaussian:
\begin{gather}
    p(\vx_0 | \vx_t) \approx \mathcal N(\E[\vx_0 | \vx_t], V[\vx_0 | \vx_t]).
\end{gather}
For linear inverse problems, this leads to the following approximation for the measurements' score:
\begin{gather}
    \nabla\log p(\vy | \vX_t=\vx_t) \nonumber \\ 
    \approx -\eqnmarkbox[Plum]{proj}{\nabla_{\vx_t}\E[\vx_0 | \vx_t]^{\top}A^{\top}(\sigma_{\vy}^2 I + A\sigma_t^2 \nabla_{\vx_t}\E[\vx_0 | \vx_t]A^{\top})^{-1}}(\eqnmarkbox[NavyBlue]{merror}{\vy - A\E[\vx_0 | \vx_t]}).
\end{gather}
\annotate[yshift=-0.5em]{below,right}{merror}{measurements error}
\annotate[yshift=-0.5em]{below,left}{proj}{``lifting'' matrix}
\vspace{1em}

In high-dimensions, even materializing the matrix $\nabla_{\vx_t}\E[\vx_0|\vx_t]$ is computationally intensive. Instead, the authors of \citep{rozet2024learning} use automatic differentiation to compute the Jacobian-vector products.

\subsubsection{BlindDPS~\citet{blinddps}}

Methods that were considered so far were designed for non-blind inverse problems, where $A$ is fully known. BlindDPS targets the case where we have a parametrized {\em unknown} forward model $A_\vphi$ (e.g. blurring with an unknown kernel $\vphi$). In BlindDPS, on top of the posterior mean approximation of $\vx_t$, one approximates the parameter of the forward model, again, with the posterior mean. Specifically, we design two parallel generative SDEs
\begin{align}
    \dx_t &= \left(\vf(\vx_t, t) - g^2(t)\nabla_{\vx_t}\log p_t(\vx_t, \vphi_t|\vy)\right)\dt + g(t)\dW_t \\
    \mathrm{d}\vphi_t &= \left(\vf(\vphi_t, t) - g^2(t)\nabla_{\vphi_t}\log p_t(\vx_t, \vphi_t|\vy)\right)\dt + g(t)\dW_t,
\end{align}
where the two SDEs are coupled through $\log p_t(\vx_t, \vphi_t|\vy)$, where under the independence between $\vX_t$ and $\vPhi_t$, the Bayes rule reads
\begin{align}
    \nabla_{\vx_t} \log p_t(\vx_t, \vphi_t|\vy) &= \nabla_{\vx_t} \log p_t(\vx_t) + \nabla_{\vx_t} \log p_t(\vy|\vX_t=\vx_t, \vPhi_t=\vphi_t) \\
    \nabla_{\vphi_t} \log p_t(\vx_t, \vphi_t|\vy) &= \nabla_{\vphi_t} \log p_t(\vphi_t) + 
    \nabla_{\vx_t} \log p_t(\vy|\vX_t=\vx_t, \vPhi_t=\vphi_t),
\end{align}
where we see that $\vX_t$ and $\vPhi_t$ are coupled through the likelihood $p(\vy|\vX_t, \vPhi_t)$.
In BlindDPS, the approximation used in DPS is applied to both the image and the operator, leading to
\begin{align}
    p(\vy|\vX_t=\vx_t, \vPhi_t=\vphi_t) \approx p(\vy|\vX_0 = \E[\vX_0|\vX_t = \vx_t], \vPhi_0 = \E[\vPhi_0|\vPhi_t = \vphi_t]).
\end{align}
The gradient of the coupled likelihood with respect to $\vx_t$ leads to
\begin{gather}
    \nabla_{\vx_t} \log p(\vy | \vX_t=\vx_t, \vPhi_t=\vphi_t) \nonumber \\ \approx -
    \eqnmarkbox[OliveGreen]{guidance}{\frac{1}{2\sigma_{\vy}^2}}  
    \eqnmarkbox[Plum]{proj}{\nabla_{\vx_t}^{\top}\E[\vX_0|\vX_t = \vx_t]A_{\E[\vPhi_0|\vPhi_t = \vphi_t]}^{\top}}
    \left(
    \eqnmarkbox[NavyBlue]{merror}{\vy - A_{\E[\vPhi_0|\vPhi_t = \vphi_t]}\E[\vX_0|\vX_t = \vx_t]}
    \right).
\end{gather}
\annotate[yshift=0.5em]{above}{merror}{measurements error}
\annotate[yshift=0.5em]{above,left}{proj}{``lifting'' matrix}
\annotate[yshift=0.25em,xshift=-1em]{above,left}{guidance}{guidance strength}
Similarly, for $\vphi_t$, we have
\begin{gather}
    \nabla_{\vphi_t} \log p(\vy | \vX_t=\vx_t, \vPhi_t=\vphi_t) \nonumber \\ \approx -
    \eqnmarkbox[OliveGreen]{guidance}{\frac{1}{2\sigma_{\vy}^2}}  
    \eqnmarkbox[Plum]{proj}{\nabla_{\vphi_t}^{\top}\E[\vX_0|\vX_t = \vx_t]A_{\E[\vPhi_0|\vPhi_t = \vphi_t]}^{\top}}
    \left(
    \eqnmarkbox[NavyBlue]{merror}{\vy - A_{\E[\vPhi_0|\vPhi_t = \vphi_t]}\E[\vX_0|\vX_t = \vx_t]}
    \right).
\end{gather}
\annotate[yshift=0.5em]{above}{merror}{measurements error}
\annotate[yshift=0.5em]{above,left}{proj}{``lifting'' matrix}
\annotate[yshift=0.25em,xshift=-1em]{above,left}{guidance}{guidance strength}

\subsubsection{DDRM Family}

The methods under the DDRM family poses all linear inverse problems to a noisy inpainting problem, by decomposing the measurement matrix with singular value decomposition (SVD), i.e. $A = U\Sigma V^\top$, where $U \in \R^{m \times m}, V \in \R^{n \times n}$ are orthogonal matrices, and $\Sigma \in \R^{m \times n}$ is a rectangular diagonal matrix with singular values $\{s_j\}_{j=1}^m$ as the elements. One can then rewrite $\vy = A\vx + \sigma_{\vy}\vz,\, \vz \sim \gN(\mathbf{0}, I_m)$ as
\begin{gather}
\label{eq:spectral_inpainting}
    \bar\vy = \Sigma\bar\vx + \sigma_{\vy}\bar\vz, \quad \mbox{where} \quad \bar\vy := U^\top \vy,\, \bar\vx := V^\top\vx,\, \bar\vz := U^\top\vz.
\end{gather}
Subsequently, \Eqref{eq:spectral_inpainting} becomes an inpainting problem in the spectral space.

\paragraph{SNIPS~\citep{kawar2021snips}.} SNIPS proceeds by first solving the inverse problem posed as \Eqref{eq:spectral_inpainting} in the spectral space to achieve a sample $\bar\vx \sim p(\bar\vx|\bar\vy)$, then retrieving the posterior sample with $\hat\vx = V\bar\vx$. The key approximation can be concisely represented as
\begin{align}
\label{eq:snips}
    \nabla_{\bar\vx_t} \log p(\bar\vy|\bar\vX_t = \vx_t) \approx 
    -\eqnmarkbox[Plum]{proj}{\Sigma^\top \left|\sigma_{\vy}^2 I_m - \sigma_t^2\Sigma\Sigma^\top\right|^\dagger}(
    \eqnmarkbox[NavyBlue]{merror}{\bar\vy - \Sigma\bar\vx_t}
    ),
\end{align}
\annotate[yshift=-0.5em]{below,right}{merror}{measurements error}
\annotate[yshift=-0.5em]{below,left}{proj}{``lifting'' matrix}
\vspace{1em}

For the simplest case of denoising where $m = n$ and $\Sigma = A = I$, the method becomes~\citep{kawar2021stochastic}
\begin{align}
    \nabla_{\vX_t} \log p(\vy|\vX_t = \vx_t) \approx \frac{\vy - \vx_t}{|\sigma_{\vy}^2 - \sigma_t^2|}.
\label{eq:kawar_posterior_denoising}
\end{align}
which produces a vector direction that is weighted by the absolute difference between the diffusion noise level $\sigma_t^2$, and the measurement noise level $\sigma_{\vy}^2$. For the fully general case in \Eqref{eq:snips}, elements in different indices are weighted according to the singular values contained in $\Sigma$. In practice, SNIPS uses pre-conditioning with the approximate negative inverse Hessian of $\log p(\bar\vx_t|\bar\vy)$ when running annealed Langevin dynamics.

\paragraph{DDRM~\citep{ddrm}.} DDRM extends SNIPS by leveraging the posterior mean $\bar\vx_{0|t} := V\E[\vX_0|\vX_t=\vx_t]$ in the place of $\bar\vx_t$ used in SNIPS. i.e.,
\begin{align}
\label{eq:ddrm_approx}
    \nabla_{\bar\vx_t} \log p(\bar\vy|\bar\vX_t = \vx_t) \approx 
    -\eqnmarkbox[Plum]{proj}{\Sigma^\top \left|\sigma_{\vy}^2 I_m - \sigma_t^2\Sigma\Sigma^\top\right|^\dagger}
    (
    \eqnmarkbox[NavyBlue]{merror}{\bar\vy - \Sigma\bar\vx_{0|t}}
    ).
\end{align}
\annotate[yshift=-0.5em]{below,right}{merror}{measurements error}
\annotate[yshift=-0.5em]{below,left}{proj}{``lifting'' matrix}
\vspace{1em}

Expressing \Eqref{eq:ddrm_approx} element-wise, we get
\begin{align}
    p(\bar\vx_t^{(i)}|\vX_{t+1} = \vx_{t+1},\vy) = 
    \begin{cases}
        \gN(\bar\vx_t^{(i)}; \bar\vx_{0|t+1}^{(i)}, \sigma_t^2) & \mbox{if  } s_i = 0\\
        \gN(\bar\vx_t^{(i)}; \bar\vx_{0|t+1}^{(i)}, \sigma_t^2) & \mbox{if  } \sigma_t < \frac{\sigma_{\vy}}{s_i}\\
        \gN(\bar\vx_t^{(i)}; \bar\vy^{(i)}, \sigma_t^2 - \frac{\sigma_{\vy}^2}{s_i^2}) & \mbox{if  } \sigma_t \geq \frac{\sigma_{\vy}}{s_i} \\
    \end{cases},
\label{eq:ddrm_simple}
\end{align}
where $\vx^{(i)}$ denotes the $i$-th element of the vector, and $s_i$ its corresponding singular value. Here, DDRM introduces another hyper-parameter $\eta$ to control the stochasticity of the sampling process
\begin{align}
    p(\bar\vx_t^{(i)}|\vX_{t+1} = \vx_{t+1},\vy) = 
    \begin{cases}
        \gN(\bar\vx_t^{(i)}; \bar\vx_{0|t+1}^{(i)} + \sqrt{1 - \eta^2}\sigma_t\frac{\bar\vx_{t+1}^{(i)} - \bar\vx_{0|t+1}^{(i)}}{\sigma_{t+1}}, \eta^2\sigma_t^2) & \mbox{if  } s_i = 0\\
        \gN(\bar\vx_t^{(i)}; \bar\vx_{0|t+1}^{(i)} + \sqrt{1 - \eta^2}\sigma_t\frac{\bar\vy^{(i)} - \bar\vx_{0|t+1}^{(i)}}{\sigma_{\vy}/s_i}, \eta^2\sigma_t^2) & \mbox{if  } \sigma_t < \frac{\sigma_{\vy}}{s_i}\\
        \gN(\bar\vx_t^{(i)}; \bar\vy^{(i)}, \sigma_t^2 - \frac{\sigma_{\vy}^2}{s_i^2}) & \mbox{if  } \sigma_t \geq \frac{\sigma_{\vy}}{s_i} \\
    \end{cases},
\label{eq:ddrm}
\end{align}
with $\eta \in (0, 1]$ such that $\eta = 1.0$ recovers \Eqref{eq:ddrm_simple}.

\paragraph{GibbsDDRM.}
GibbsDDRM~\citet{murata2023gibbsddrm} extends DDRM to the following blind linear problem
    $\vy = A_{\bm{\varphi}}\vx + \sigma_{\vy}\vz$, where $A_{\bm{\varphi}}$ is a linear operator parameterized by $\bm{\varphi}$. Here, $A_{\bm{\varphi}} = U_{\bm{\varphi}}\Sigma_{\bm{\varphi}} V_{\bm{\varphi}}^\top$ has a ${\bm{\varphi}}$ dependence SVD decomposition with singular values $\{s_{j,\bm{\varphi}}\}_{j=1}^m$ as the elements of the diagonal matrix $\Sigma_{\bm{\varphi}}$. In the spectral space, $\bar\vy_{\bm{\varphi}} := U_{\bm{\varphi}}^\top \vy_{\bm{\varphi}},\, \bar\vx_{\bm{\varphi}} := V_{\bm{\varphi}}^\top\vx_{\bm{\varphi}},\, \bar\vz_{\bm{\varphi}} := U_{\bm{\varphi}}^\top\vz_{\bm{\varphi}}$.
    Subsequently, the posterior mean in DDRM is replaced with $\bar\vx_{0|t, \bm{\varphi}} := V_{\bm{\varphi}}\E[\vX_0|\vX_t=\vx_t]$, also depending on $\bm{\varphi}$. Thus, it leads to the sampling process

\begin{align}
    p(\bar\vx_{t,\bm{\varphi}}^{(i)}|\vX_{t+1} = \vx_{t+1},\vy,\bm{\varphi}) = 
    \begin{cases}
        \gN(\bar\vx_{t,\bm{\varphi}}^{(i)}; \bar\vx_{0|t+1, \bm{\varphi}}^{(i)} + \sqrt{1 - \eta^2}\sigma_t\frac{\bar\vx_{t+1, \bm{\varphi}}^{(i)} - \bar\vx_{0|t+1, \bm{\varphi}}^{(i)}}{\sigma_{t+1}}, \eta^2\sigma_t^2) & \mbox{if  } s_{i,\bm{\varphi}} = 0\\
        \gN(\bar\vx_{t,\bm{\varphi}}^{(i)}; \bar\vx_{0|t+1, \bm{\varphi}}^{(i)} + \sqrt{1 - \eta^2}\sigma_t\frac{\bar\vy^{(i)}_{\bm{\varphi}} - \bar\vx_{0|t+1, \bm{\varphi}}^{(i)}}{\sigma_{\vy}/s_{i,\bm{\varphi}}}, \eta^2\sigma_t^2) & \mbox{if  } \sigma_t < \frac{\sigma_{\vy}}{s_{i,\bm{\varphi}}}\\
        \gN(\bar\vx_{t,\bm{\varphi}}^{(i)}; \bar\vy^{(i)}_{\bm{\varphi}}, \sigma_t^2 - \frac{\sigma_{\vy}^2}{s_{i,\bm{\varphi}}^2}) & \mbox{if  } \sigma_t \geq \frac{\sigma_{\vy}}{s_{i,\bm{\varphi}}} \\
    \end{cases}.
\label{eq:gibbsddrm}
\end{align}
 At time step $t$, $\bm{\varphi}$ is sampled by using the  conditional distribution $p(\bm{\varphi}|\vx_{t:T},\vy)$ and updated for several iterations in a Langevin manner:
 \begin{align*}
     \bm{\varphi} \leftarrow \bm{\varphi} + \frac{\xi}{2}\nabla_{\bm{\varphi}}\log p(\bm{\varphi}|\vx_{t:T},\vy) + \sqrt{\xi} \bm{\epsilon},
 \end{align*}
where $\xi$ is a stepsize and $\bm{\epsilon}\sim\mathcal{N}(\bm{0},I_n)$. Here, $\nabla_{\bm{\varphi}}\log p(\bm{\varphi}|\vx_{t:T},\vy)\approx\nabla_{\bm{\varphi}}\log p(\bm{\varphi}|\bar\vx_{0|t, \bm{\varphi}},\vy)$, and the gradient can be computed as:
\begin{align}
    \nabla_{\bm{\varphi}}\log p(\bm{\varphi}|\bar\vx_{0|t, \bm{\varphi}},\vy)=-\frac{1}{2\sigma^2_{\vy}}\nabla_{\bm{\varphi}}\norm{\vy - \mathbf{A}_{\bm{\varphi}}\bar\vx_{0|t, \bm{\varphi}}}.
\end{align}

\subsubsection{DDNM~\citet{ddnm} family}
A different way to find meaningful approximations for the conditional score is to look at the conditional version of Tweedie's formula, see \Eqref{eq:cond_tweedie}.
Using Bayes rule and rearranging~\citet{ravula2023optimizing}, we have
\begin{align}
    \E[\vX_0 | \vX_t=\vx_t, \vy] &= \vx_t + \sigma_t^2 \nabla_{\vx_t} \log p_t(\vX_t|\vy) \\
    &= \vx_t + \sigma_t^2 \nabla_{\vx_t} \log p_t(\vx_t) + \sigma_t^2 \nabla_{\vx_t} \log p_t(\vy|\vX_t=\vx_t) \\
    &= \E[\vX_0 | \vX_t=\vx_t] + \sigma_t^2 \nabla_{\vx_t} \log p_t(\vy|\vX_t=\vx_t).
\label{eq:cond_tweedie_score}
\end{align}
The methods that belong to the DDNM family make approximations to $\E[\vX_0|\vX_t=\vx_t,\vy]$ by making certain data consistency updates to $\E[\vX_0|\vX_t=\vx_t]$.
\paragraph{DDNM~\citet{ddnm}.}
The simplest form of update when considering no noise can be obtained through range-null space decomposition, assuming that one can compute the pseudo-inverse. In DDNM, this condition is trivially met by considering operations that are SVD-decomposable. DDNM proposes to use the following projection step to the posterior mean to obtain an approximation of the conditional posterior mean
\begin{align}
    \E[\vX_0 | \vX_t=\vx_t, \vy] \approx (I - A^\dagger A)\E[\vX_0 | \vX_t=\vx_t] + A^\dagger \vy,
\label{eq:ddnm}
\end{align}
where $A^\dagger$ is the Moore-Penrose pseudo-inverse of $A$. 
One can also express \Eqref{eq:ddnm} as an approximation of the likelihood, consistent to other methods in the chapter. Specifically,
notice that by using the relation in \Eqref{eq:cond_tweedie_score},
\begin{align}
\label{eq:cond_score_tweedie}
    \nabla_{\vx_t}\log p_t(\vy|\vX_t=\vx_t) = \frac{1}{\sigma_t^2}(\E[\vX_0 | \vX_t=\vx_t, \vy] - \E[\vX_0 | \vX_t=\vx_t]).
\end{align}
Plugging in \Eqref{eq:ddnm} to \Eqref{eq:cond_score_tweedie},
\vspace{1em}
\begin{align}
    \nabla_{\vx_t}\log p_t(\vy|\vX_t=\vx_t) \approx
    \eqnmarkbox[OliveGreen]{guidance}{-\frac{1}{\sigma_t^2}}
    \eqnmarkbox[Plum]{proj}{A^\dagger}
    \left(
    \eqnmarkbox[NavyBlue]{merror}{\vy - A\E[\vX_0 | \vX_t=\vx_t, \vy]}
    \right)
\label{eq:ddnm_label}
\end{align}
\annotate[yshift=0.5em]{above}{merror}{measurements error}
\annotate[yshift=0.5em]{above,left}{proj}{``lifting'' matrix}
\annotate[yshift=0.25em,xshift=-1em]{below,left}{guidance}{guidance strength}
\vspace{1em}

When there is noise in the measurement, one can make soft updates
\begin{align}
    \E[\vX_0 | \vX_t=\vx_t, \vy] \approx (I - \Sigma_t A^\dagger A)\E[\vX_0 | \vX_t=\vx_t] + \Sigma_t A^\dagger \vy, \quad \Sigma \in \R^{n \times n}.
\end{align}
Also, similar to \Eqref{eq:ddnm_label},
\vspace{1em}
\begin{align}
    \nabla_{\vx_t}\log p_t(\vy|\vX_t=\vx_t) \approx 
    \eqnmarkbox[OliveGreen]{guidance}{-\frac{1}{\sigma_t^2}}
    \eqnmarkbox[Plum]{proj}{\Sigma_t A^\dagger}
    \left(
    \eqnmarkbox[NavyBlue]{merror}{\vy - A\E[\vX_0 | \vX_t=\vx_t, \vy]}
    \right)
\label{eq:ddnm_label_noisy}
\end{align}
\annotate[yshift=0.5em]{above}{merror}{measurements error}
\annotate[yshift=0.5em]{above,left}{proj}{``lifting'' matrix}
\annotate[yshift=0.25em,xshift=-1em]{below,left}{guidance}{guidance strength}
\vspace{1em}

Here, one can choose a simple $\Sigma_t = \lambda_t I$ with $\lambda_t$ set as a hyper-parameter, or use different scaling for each spectral component. Observe that due to the relationship between the (conditional) score function and the posterior mean established in \Eqref{eq:cond_tweedie_score}, we can also easily rewrite the approximation in terms of the score of the posterior.
\paragraph{DDS~\citet{dds}, DiffPIR~\citet{diffpir}.}
Both DDS and DiffPIR propose a proximal update to approximate the conditional posterior mean, albeit from different motivations. The resulting approximation reads
\begin{align}
    \E[\vX_0 | \vX_t=\vx_t, \vy] \approx \argmin_{\vx} \frac{1}{2} \|\vy - A\vx\|^2 + \frac{\lambda_t}{2} \|\vx - \E[\vX_0|\vX_t=\vx_t]\|^2.
\label{eq:proximal_tweedie}
\end{align}
The difference between the two algorithms comes from how one solves the optimization problem in \Eqref{eq:proximal_tweedie}, and how one chooses the hyperparameter $\lambda_t$. In DDS, the optimization is solved with a few-step conjugate gradient (CG) update steps, by showing that DPS gradient update steps can be effectively replaced with the CG steps under assumptions on the data manifold~\citet{dds}. $\lambda_t$ is taken to be a constant value across all $t$. DiffPIR uses a closed-form solution for \Eqref{eq:proximal_tweedie}, and proposes a schedule for $\lambda_t$ that is proportional to the signal-to-noise (SNR) ratio of the diffusion at time $t$. Specifically, one chooses $\lambda_t = \sigma_t\zeta$, where $\zeta$ is a constant.

\subsection{Variational Inference}
These methods approximate the true posterior distribution, $p(\vx | \vy)$, with a simpler, tractable distribution. Variational formulations are then used to optimize the parameters of this simpler distribution. 

\subsubsection{RED-Diff~\citet{mardani2024variational}}

\citet{mardani2024variational} introduce RED-diff, a new approach for solving inverse problems by leveraging stochastic optimization and diffusion models. The core idea is to use variational method by introducing a simpler distribution, $q :=\mathcal{N}(\bm{\mu},\sigma^2 I_n)$, to approximate the true posterior  $p(\vx_0|\vy)$ by minimizing the KL divergence $\mathcal D_{\mathrm{KL}}$ between them:
\begin{align}
\min_{q} \mathcal D_{\mathrm{KL}}(q(\vx_0|\vy)\Vert p(\vx_0|\vy)).
\label{eq:var_inference}
\end{align}
Here, $\mathcal D_{\mathrm{KL}}(q(\vx_0|\vy)\Vert p(\vx_0|\vy))$ can be written as follows:
\begin{align}
    \mathcal D_{\mathrm{KL}}(q(\vx_0|\vy)\Vert p(\vx_0|\vy)) = \underbrace{-\mathbb{E}_{q(\vx_0|\vy)}[\log p(\vy|\vx_0)] + \mathcal D_{\mathrm{KL}}\big(q(\vx_0|\vy)\Vert p(\vx_0)\big)}_{\text{Variational Bound (VB)}} + \mathrm{constant}.
\end{align}
via classic variational inference argument. The first term in VB can be simplified into reconstruction loss, and the second term can be decomposed as score-matching objective which involves matching the score function of the variational distribution with the score function of the true posterior denoisers at different timesteps:
\begin{align}
\min_{\mu} \frac{||\vy-\mathcal A(\bm{\mu})||^2}{2\sigma_\vy^2} + \E_{t,\bm{\epsilon}}[\lambda_t||\epsilon_{\bm{\theta}}(\vx_t;t)-\bm{\epsilon}||^2]
\label{eq:REDdiff}
\end{align}
where $\bm{\mu}$ is the mean of the variational distribution,  
and $\sigma_v^2$ is the noise variance in the observation, 
$\bm{\epsilon}_{\bm{\theta}}(x_t;t)$ is the score function of the diffusion model at timestep ($t$) and $\lambda_t$ is a time-weighting factor.

\noindent \textbf{Sampling as optimization.} The goal is then to find an image $\bm{\mu}$ that reconstructs the observation $\vy$ given by $f$, while having a high likelihood under the denoising diffusion prior (regularizer). This score-matching objective is optimized using stochastic gradient descent, effectively turning the sampling problem into an optimization problem. The weighting factor ($\lambda_t$) is chosen based on the signal-to-noise ratio (SNR) at each timestep to balance the contribution of different denoisers in the diffusion process.

\subsubsection{Blind RED-Diff~\citet{alkan2023variational}}
In~\citet{alkan2023variational} authors introduce blind RED-diff, an extension of the RED-diff framework~\citet{mardani2024variational} to solve blind inverse problems. The main idea is to use variational inference to jointly estimate the latent image and the unknown forward model parameters.

Similar to RED-Diff, the key mathematical formulation is the minimization of the KL-divergence between the true posterior distribution $p(\vx_0, \gamma|\vy)$ and a variational approximation $q(\vx_0, \gamma | \vy)$:
$$
\min_q \mathcal D_{\mathrm{KL}}(q(\vx_0, \gamma|\vy)\|p(\vx_0, \gamma|\vy)).
$$
If we assume the latent image and the forward model parameters are independent, the KL-divergence can be decomposed as:
$$
\mathcal D_{\mathrm{KL}}(q(\vx_0|\vy)||p(\vx_0)) + \mathcal D_{\mathrm{KL}}(q(\gamma|\vy)\|p(\gamma)) - \E_{q(\vx_0,\gamma|\vy)}[\log p(\vy|\vx_0, \gamma)] + \log p(\vy).
$$
The minimization with respect to $q$ involves three terms:
\begin{enumerate}
\item[i.] $\mathcal D_{\mathrm{KL}}(q(\vx_0|\vy)\|p(\vx_0))$ represents the KL divergence between the variational distribution of the image ($\vx_0$) and its prior distribution. This term is approximated using a score-matching loss, which leverages denoising score matching with a diffusion model (as in RED-Diff).

\item[ii.] $\mathcal D_{\mathrm{KL}}(q(\gamma|\vy)\|p(\gamma))$ is the KL divergence between the variational distribution of the forward model parameters ($\gamma$) and their prior distribution. This term acts as a \emph{regularizer} on $\gamma$.

\item[iii.] $-\E_{q(\vx_0,\gamma|\vy)}[\log p(\vy|\vx_0, \gamma)]$ is the expectation of the negative log-likelihood of the observed data $\vy$ given the image $\vx_0$ and the forward model parameters $\gamma$. This term ensures data consistency.
\end{enumerate}

The resulting optimization can be achieved using alternating stochastic optimization, where the image $x_0$ and the forward model parameters $\gamma$ are updated iteratively.

The formulation assumes conditional independence between $\vx_0$ and $\gamma$ given the measurement $\vy$, and it also requires a specific form for the prior distribution $p(\gamma)$.

\subsubsection{Score Prior~\citet{feng2023score}}
We again start by introducing a variational distribution $q_{\bm{\phi}}(\vx_0)$ that aims to approximate the posterior distribution determined by the diffusion prior. The optimization problem becomes
\begin{align}
\label{eq:vi}
    &\min_{\bm{\phi}} \mathcal D_{\mathrm{KL}}(q_{\bm{\phi}}(\vx_0)||p_{\bm{\theta}}(\vx_0|\vy)) \\
    &\min_{\bm{\phi}} \int q_{\bm{\phi}}(\vx_0|\vy)[-\log p(\vy|\vx_0) - \log p_{\bm{\theta}}(\vx_0) + \log q_{\bm{\phi}}(\vx_0)].
\end{align}
One of the most expressive yet tractable proposal distributions is normalizing flows (NF)~\citet{rezende2015variational,dinh2016density}. Choosing $q_{\bm{\phi}}$ to be an NF, we can transform the optimization problem to
\begin{align}
\label{eq:berthy}
    \min_{\bm{\phi}} \E_{\vz \sim \mathcal{N}(\mathbf{0}, I_n)}\left[
    \underbrace{-\log p(\vy|G_{\bm{\phi}}(\vz))}_{\rm Likelihood} - \underbrace{\log p_{\bm{\theta}}(G_{\bm{\phi}}(\vz))}_{\rm Prior} + \underbrace{\log \pi(\vz) - \log \left|\det \frac{dG_{\bm{\phi}}(\vz)}{d\vz}\right|}_{\rm Entropy}
    \right],
\end{align}
where the expectation is over the input latent variable $\vz$, and $\pi$ is the reference Gaussian distribution. Observe that the likelihood term and the entropy can be efficiently computed with a single forward/backward pass through the NF due to the parametrization of $q_{\bm{\phi}}$ with an NF. All that is left for us is to compute the prior term $\log p_{\bm{\theta}}(G_{\bm{\phi}}(\vz))$. In score prior~\citet{feng2023score}, this is solved by leveraging the instantaneous change-of-variables formula with the diffusion PF-ODE, as originally proposed in~\citet{ncsnv3}
\begin{align}
\label{eq:likelihood_comp_pfode}
    \log p_{\bm{\theta}} (\vx_0) = \log p_T (\vx_T) + \int_0^T \nabla \cdot \tilde \vf_{\bm{\theta}}(\vx_t, t)\, \d t,
\end{align}
where $\vf_{\bm{\theta}}(\vx_t, t)$ is the drift term of the reverse SDE in \Eqref{eq:backward_process} with the score replaced by the network approximation. Notice that by plugging in \Eqref{eq:likelihood_comp_pfode} to \Eqref{eq:berthy}, we can optimize the NF model in an unsupervised fashion. Notice that while this formulation does not incur approximation errors, it is very costly as every optimization steps involve computing \Eqref{eq:likelihood_comp_pfode}. Moreover, observe that the training of NF is done for a specific measurement $\vy$. One has to run \Eqref{eq:berthy} for every different measurement that one wishes to recover.
\subsubsection{Efficient Score Prior~\citet{feng2023efficient}}
As computing \Eqref{eq:likelihood_comp_pfode} is costly, Feng {\em et al.~} proposed to optimize $q_{\bm{\phi}}$ with the evidence lower bound (ELBO), originally presented in the work of Score-flow~\citet{song2021maximum} $b_{\bm{\theta}}(\vx_0) \leq \log p_{\bm{\theta}}(\vx_0)$
\begin{align}
\label{eq:elbo_likelihood}
    b_{\bm{\theta}}(\vx_0) = \E_{p(\vx_T|\vx_0)}\left[ \log \pi(\vx_T) \right] - \frac{1}{2} \int_0^T g(t)^2 h(t)\,\d t,
\end{align}
where
\begin{align}
\label{eq:eff_h}
    h(t) := \E_{p(\vx_t|\vX_0)} \Bigg[
    \underbrace{\frac{1}{\sigma_t^4}\|\vh_{\bm{\theta}}(\vx_t) - \vx_0\|_2^2}_{\mathrm{Denoising \,\,loss}}
    - \| \nabla_{\vx_t} \log p(\vx_t|\vx_0) \|_2^2 - \frac{2}{g(t)^2} \nabla_{\vx_t} \cdot \vf(\vx_t, t)
    \Bigg].
\end{align}
Intuitively, the value of $b_{\bm{\theta}}$ is small when we have a small denoising loss, and large when our diffusion denoiser $\vh_{\bm{\theta}}$ cannot properly denoise the given image. Replacing the exact likelihood \Eqref{eq:likelihood_comp_pfode} that requires hundreds to thousands of NFEs to the surrogate denoising likelihood \Eqref{eq:elbo_likelihood} that requires only a single NFE makes the method much more efficient and scalable to higher dimensions.

\subsection{Asymptotically Exact Methods}
\label{sec:mc_methods}
These methods aim to sample from the true posterior distribution. Of course, the intractability of the posterior distribution cannot be circumvented but what these methods trade compute for approximation error: as the number of network evaluations increases to infinity, these methods will asymptotically converge to the true posterior (assuming no other approximation errors).

\subsubsection{Plug and Play Diffusion Models (PnP-DM)~\citep{wu2024principled}}
As explained in the introduction, the end goal is to sample from the distribution $p(\vx_0 | \vy) \propto p(\vx_0) p(\vy | \vx)$. The authors of~\citep{wu2024principled} introduce an auxiliary variable $\vz$ and an auxiliary distribution: 
\begin{gather}
    \pi(\vx_0, \vz | \vy) \propto p(\vx_0)\cdot p(\vy | \vz)\cdot \mathrm{exp}\left(-\frac{1}{2\rho^2}||\vx_0 -\vz||^2\right).
\end{gather}
It is easy to see that as $\rho \to 0$, the auxiliary distribution 
converges to the target distribution $p(\vx_0 | \vy)$.

To sample from the joint distribution $\pi(\vx_0, \vz|y)$, the authors use Gibbs Sampling, i.e. the alternate between sampling from the posteriors. Specifically, the sampling algorithm alternates between two steps:
\begin{itemize}
    \item Likelihood term:
    \begin{gather}
        \vz^{(k)} \sim \pi(\vz | \vy, \vx_0^{(k)}) \propto p(\vy | \vz)\cdot \mathrm{exp}\left(-\frac{1}{2\rho^2}||\vx_0^{(k)} -\vz||^2\right).
        \label{eq:sample_something_that_matches_measurements}
    \end{gather}
    \item Prior term:
    \begin{gather}
        \vx_0^{(k+1)}\sim \pi(\vx_0 | \vy, \vz^{(k)}) \propto p(\vx_0) \cdot \mathrm{exp}\left(-\frac{1}{2\rho^2}||\vx_0^{(k)} -\vz^{(k)}||^2\right).   
        \label{eq:prior_term_pnp}
    \end{gather}
\end{itemize}
The likelihood term samples a vector that satisfies the measurements and is close to $\vx_0^{(k)}$. The prior term samples a vector that is likely under $p(\vx_0)$ and is close to $\vz^{(k)}$. For most problems of interest, sampling from \Eqref{eq:sample_something_that_matches_measurements} is easy because the distribution is log-concave, e.g. that's the case for linear inverse problems. The interesting observation is that sampling from \Eqref{eq:prior_term_pnp} corresponds to a denoising problem, for which diffusion models excel. Indeed, for any $\vx_t$ at noise level $\sigma_t$, we have that:
\begin{gather}
    p(\vx_0|\vx_t) \propto  p(\vx_0) p(\vx_t | \vx_0) =  p(\vx_0) \mathrm{exp}\left(-\frac{1}{2\sigma_t^2}||\vx_0 - \vx_t||^2\right).
\end{gather}
Hence, to sample from \Eqref{eq:prior_term_pnp}, one initializes the reverse process at $\vz^{(k)}$ and time $t$ such that: $\sigma_t = \rho$.

\subsubsection{FPS~\citet{dou2023diffusion}}
FPS connects posterior sampling to Bayesian filtering and uses sequential Monte Carlo methods to solve the filtering problem, avoiding the need to handcraft approximations to the posterior  $p(\vy|\vx_t)$. Given an observation $\vy$, FPS proposes to first construct a sequence $\{\vy_t\}_{t=0}^N$ from $\vy$, and then determine a tractable distribution $p(\vx_{t-1}\vert \vx_{t}, \vy_{t-1})$. Starting from $\vx_N \sim \mathcal{N}(\mathbf{0}, I_n)$, FPS can then recursively sample $\vx_t$ for $t = N-1,\dots,1$, and finally obtain $\vx_0$. Specifically, FPS consists of two steps: 

\begin{enumerate}
    \item[\textbf{Step 1.}] \textbf{Generating a sequence of $\{\vy_t\}_{t=0}^{N}$ with an observation $\vy$.} This can be done either using the forward process or unconditional DDIM backward sampling.

For the construction via the forward process, we recursively construct $\vy_t$ as follows:
    \begin{gather}
    \vy_t=\vy_{t-1}+\sigma_{t} A\vz_t, \quad \text{initialized with } \vy_0:=\vy.
    \end{gather}
    This arises from $\vx_t = \vx_{t-1} + \sigma_t \vz_t$ and applying the linear operator $A$ to it.

    For the construction via backward sampling, FPS uses methods such as unconditional DDIM as in \Eqref{eq:ddim},
    \begin{gather}
    \vy_{t-1}=\underbrace{u_t\vy_0}_{\mathrm{clean}} +\underbrace{v_t\vy_{t}}_{\mathrm{direction\,\,to\,\,time\,\,}t \mathrm{\,\,sample}}+\underbrace{w_t A\vz_t}_{\mathrm{noise}}, \quad \text{initialized with } \vy_N\sim\mathcal{N}(\mathbf{0}, AA^{\top}). 
\end{gather}
Here, $u_t$, $v_t$, and $w_t$ are DDIM coefficients that can be explicitly computed. Note that $\vy_N$ is sampled from $\mathcal{N}(\mathbf{0}, AA^\top)$ because the prior distribution of the diffusion model is a standard Gaussian $\vx_N \sim \mathcal{N}(\mathbf{0}, I)$, and due to the linearity of the inverse problem, $\vy_N = A\vx_N$.
    \item[\textbf{Step 2.}] \textbf{Generating a backward sequence of $\{\vx_t\}_{t=0}^{N}$ from Step 1's $\{\vy_t\}_{t=0}^{N}$.} First, note that $p(\vx_{t-1}\vert \vx_{t}, \vy_{t-1})$ is a tractable normal distribution. This results from applying Bayes' rule and the conditional independence of  $\vx_t$ and the random vector $\vY_{t-1}$ given $\vx_{t-1}$:
    \begin{gather}
    p(\vx_{t-1}\vert \vx_{t}, \vY_{t-1})\propto p(\vx_{t-1}\vert \vx_{t}) p(\vY_{t-1}\vert \vx_{t-1}).
\end{gather}
 Here, $p(\vx_{t-1}\vert \vx_{t})$ is approximated via backward diffusion sampling with learned scores, and $p(\vY_{t-1}\vert \vx_{t-1})=\mathcal{N}(A\vx_{t-1}, c_{t-1}^2I)$, where $c_{t-1}$, dependent on $\sigma_{\vy}>0$, can be computed explicitly~\citep{song2021solving}.  Thus, with $\{\vy_t\}_{t=0}^{N}$ and initial condition $\vx_N\sim\mathcal{N}(\mathbf{0}, I_n)$, FPS recursively samples $\vx_{N-1},\cdots \vx_{1}$ using $p(\vx_{t-1}\vert \vx_{t}, \vY_{t-1}=\vy_{t-1})$, ultimately yielding $\vx_0$. 
\end{enumerate}

FPS algorithm is theoretically supported to recover the oracle $p(\vx|\vy)$ once the step size is sufficiently small.

\subsubsection{PMC~\citet{sun2024provable}}
Plug-and-Play (PnP)~\citet{kamilov2023plug} and RED~\citet{romano2016little} are two representative methods of using denoisers as priors for solving inverse problems. Let $g_\vy(\vx) = \frac{1}{2\sigma_\vy^2} \|\vy - A\vx\|_2^2$ be the log-likelihood function, $\vh_{\bm{\theta}}^\sigma(\cdot)$ an MMSE denoiser from \Eqref{eq:dsm} conditioned on the noise level $\sigma$, and $R_{\bm{\theta}}^\sigma(\cdot) := \mathrm{Id} - \vh_{\bm{\theta}}^\sigma(\cdot)$ the residual projector. Note that conditioning on the noise level $\sigma$ is equivalent to the network being conditioned no $t$, since the mapping is one-to-one.
A single iteration of these methods read
\begin{itemize}
    \item PnP proximal gradient method~\citet{kamilov2017plug}:
    \begin{align}
        \vx_{k+1} &= \vh_{\bm{\theta}}^\sigma(\vx_k - \gamma\nabla_{\vx_k}g_\vy(\vx_k)) \\
        &= \vx_k - \gamma\left(
        \nabla_{\vx_k} g_\vy(\vx_k) + \frac{1}{\gamma} R_{\bm{\theta}}^\sigma\left(\vx_k - \gamma \nabla_{\vx_k} g_\vy(\vx_k)\right)
        \right).
    \label{eq:pmc_pnp}
    \end{align}
    \item RED gradient descent~\citet{romano2016little}:
    \begin{align}
        \vx_{k+1} &= \vx_k - \gamma\left(
        \nabla_{\vx_k}g_\vy(\vx_k) + \tau \left(
        \vx_k - \vh_{\bm{\theta}}^\sigma(\vx_k)
        \right)
        \right) \\
        &= \vx_k - \gamma\left(
        \nabla_{\vx_k} g_\vy(\vx_k) + \tau R_{\bm{\theta}}^\sigma(\vx_k)
        \right).
        \label{eq:pmc_red}
    \end{align}
\end{itemize}

Notice that by using Tweedie's formula, we see that $R_\theta^\sigma(\vx) = -\sigma^2 \nabla_\vx \log p_\sigma(\vx)$. Rearranging \Eqref{eq:pmc_pnp} and \Eqref{eq:pmc_red},

\begin{itemize}
    \item PnP:
    \begin{align}
        \frac{\vx_{k+1} - \vx_k}{\gamma} = -P(\vx_k), \quad P(\vx) := \nabla_{\vx}g_\vy(\vx) - \frac{\sigma^2}{\gamma}\nabla_{\vx} \log p_\sigma(\vx - \gamma \nabla_\vx g_\vy(\vx)),
    \label{eq:pmc_pnp2}
    \end{align}
    \item RED:
    \begin{align}
        \frac{\vx_{k+1} - \vx_k}{\gamma} = -G(\vx_k), \quad G(\vx) := \nabla_{\vx}g_\vy(\vx) - \tau\sigma^2 \nabla_{\vx} \log p_\sigma(\vx).
        \label{eq:pmc_red2}
    \end{align}
\end{itemize}
Moreover, by setting $\gamma = \sigma^2$ and $\tau = 1 / \sigma^2$, one can show that
\begin{align}
    \lim_{\sigma \rightarrow 0} P(\vx) &= \nabla_\vx g_\vy(\vx) - \lim_{\sigma \rightarrow 0} \{
    \nabla_\vx \log p(\vx - \sigma^2 \nabla_{\vx_k} g_\vy(\vx_k))
    \} \notag \\
    &= \nabla_\vx g_\vy(\vx) - \lim_{\sigma \rightarrow 0} \{
    \lim_{\sigma \rightarrow 0} \log p_\sigma(\vx)
    \} = \lim_{\sigma \rightarrow 0} G(\vx) \notag \\
    &= - \nabla_{\vx} \log p(\vy|\vx) - \nabla_{\vx} \log p(\vx) = -\nabla_{\vx} \log p(\vx|\vy).
\label{eq:pmc_convergence}
\end{align}
In other words, we see that the iteration of PnP/RED in \Eqref{eq:pmc_pnp} and \Eqref{eq:pmc_red} will converge to sampling from the posterior as $\sigma^2 = \gamma \rightarrow 0$
\begin{align}
    \d\vx_t = \nabla_{\vx_t} \log p(\vx_t|\vy)\,\dt,
\label{eq:gf_ode}
\end{align}
where $t$ indexes the continuous time flow of $\vx$, as opposed to the discrete formulations in \Eqref{eq:pmc_pnp} and \Eqref{eq:pmc_red}. Note that this notion of $t$ does not match the diffusion time $t$, where the time index matches a specific noise level. In PMC, the authors propose to incorporate noise level annealing as done in the usual reverse diffusion process by starting from a large noise level $\sigma$ and gradually reducing the noise level.
Solving \Eqref{eq:gf_ode} with PMC then boils down to iterative application of \Eqref{eq:pmc_pnp} and \Eqref{eq:pmc_red} with the annealing strategy.
Moreover, introducing Langevin diffusion yields a stochastic version
\begin{align}
    \d\vx_t = \nabla_{\vx_t} \log p(\vx_t|\vy)\,\dt + \sqrt{2}\d\vW_t,
\label{eq:gf_langevin}
\end{align}
which can be solved in the same way, but with additional stochasticity.

\subsubsection{Sequential Monte Carlo-based methods}
SMCDiff~\citet{trippe2023diffusion}, MCGDiff~\citet{cardoso2023monte}, and TDS~\citet{wu2023practical} belong to the category of sequential Monte Carlo (SMC)-based methods~\citet{doucet2001introduction}. SMC aims to sample from the posterior by constructing a sequence of distributions $\vX_{1:T}$, which terminates at the target distribution. The evolution of the distribution is approximated by $K$ particles. In a high level, SMC can be described with three steps: 1) Transition with a proposal kernel $\{\vx_{t}^{1:K}\} \sim p(\vX_t|\vX_{t-1})$, 2) computing the weights to re-weight the importance, and 3) resampling from a reweighted multinomial distribution. Methods that belong to this category propose different ways of constructing the proposal distribution and the weighting function.

\subsection{CSGM-Type methods}
\label{sec:csgm_type}
\subsubsection{DMPlug~\citep{wang2024dmplug}, SHRED~\citep{chihaoui2024zeroshot}}
Compressed sensing generative model (CSGM)~\citep{csgm,ilo} is a general method for solving inverse problems with deep generative models by aiming to find the input latent vector $\vz$ through
\begin{align}
\label{eq:csgm}
    \vz^* = \argmin_\vz \|\vy - AG_{\bm{\theta}}(\vz)\|^2,
\end{align}
where $G_{\bm{\theta}}$ is an arbitrary generative model. DMPlug and SHRED can be seen as extensions of CSGM to the case where one uses a diffusion model. Unlike GANs or Flows where the mapping from the latent space to the image space is done through a single NFE, diffusion models require multiple NFE to solve the generative SDE/ODE. One can rewrite \Eqref{eq:csgm} as
\begin{align}
\label{eq:dmplug}
    \vz^* = \argmin_\vz \|\vy - A\hat\vx(\vz)\|^2,
\end{align}
where $\hat\vx=\hat \vx(\vz)$ is the solution of the deterministic sampler initialized at $\vz$. Essentially, the models in this category optimize over the ``latent'' space of noises that are fed to the deterministic ODE sampler. One caveat of \Eqref{eq:dmplug} is the exploding memory required for backpropagation through time. To mitigate this, when sampling from $p_{\bm{\theta}}(\vx_0|\vx_T)$, a few-step sampling (e.g. 3 for DMPlug and 10 for SHRED) is used to approximate the true sampling process.

\subsubsection{CSGM with consistent diffusion models~\citep{xu2024consistency}}
Diffusion models can be distilled into one-step models, known as Consistency Models~\citep{song2023consistency}, that solve in one step the Probability Flow ODE. These models can be used in \Eqref{eq:dmplug}, replacing the ODE sampling, to reduce the computational requirements~\citep{xu2024consistency}.

\subsubsection{Intermediate Layer Optimization~\citep{ilo, score_ilo}}
CSGM has been extended to perform the optimization in some intermediate latent space~\citep{ilo}. The problem is that the intermediate latents need to be regularized to avoid exiting the manifold of realistic images. Score-Guided Intermediate Layer Optimization (Score-ILO)~\citep{score_ilo} uses diffusion models to regularize the intermediate solutions.

\subsection{Latent Diffusion Models}
\label{sec:solving_inv_with_ldms}
\subsubsection{Motivation}
In this subsection, we focus on algorithms that have been developed for solving inverse problems with latent diffusion models (see Section \ref{sec:ldms}). There are a few additional challenges when dealing with latent diffusion models that have led to a growing literature of papers that are trying to address them. 

\paragraph{Loss of linearity.} The first challenge in solving inverse problems with latent diffusion models is that linear inverse problems become essentially non-linear. The problem stems from the fact that diffusion happens in the latent space but measurements are in the pixel-space. In order to guide the diffusion there are two potential solutions: i) either project the measurements to the latent space through the encoder, or, ii) project the latents to the pixel space as we diffuse through the decoder. Both approaches depend on non-linear functions ($\mathrm{Enc}, \ \mathrm{Dec}$ respectively) and hence even linear inverse problems need a more general treatment.

\paragraph{Decoding is expensive.} The other issue that arises is computational.  Most of the time, we need to decode the latent to pixel-space to compare with the measurements. The motivation behind latent diffusion models is to accelerate training and sampling. Hence, we want to avoid repeated calls to the decoder as we solve inverse problems.

\paragraph{Decoding-encoding map is not one-to-one.} Even if we ignore the computational challenges, it is not straightforward to decode the latent to the pixel-space, compare with the measurements and get meaningful guidance in the latent space since the decoding-encoding map is not an one-to-one function.

\paragraph{Text-conditioning.} Finally, latent diffusion models typically get a textual prompt as an additional input. A lot of algorithms that have been developed in the space of using latent diffusion models to solve inverse problems innovate on how they use text conditioning.

\subsubsection{Latent DPS}

The first algorithm we review in the space of solving inverse problems with latent diffusion models is Latent DPS, i.e. the straightforward extension of DPS for latent diffusion models. The approximation made in this algorithm is:
\begin{gather}
    \nabla_{\vx^\mathrm{E}_t} \log p(\vy | \vX^\mathrm{E}_t=\vx^\mathrm{E}_t) \approx \nabla_{\vx^\mathrm{E}_t} \log p(\vy |\vX_0 = \mathrm{Dec}(\E[\vX^\mathrm{E}_0 | \vX^\mathrm{E}_t=\vx^\mathrm{E}_t])).
    \label{eq:latent_DPS_approx}
\end{gather}

The algorithm works by performing one-step denoising in the latent space and measuring how much the decoding of the denoised latent matches the measurements $\vy$.

\subsubsection{PSLD~\citet{rout2024solving}}
The performance of Latent DPS is hindered by the fact that the decoding-encoding map is not an one-to-one function, as discussed earlier. The approximation made above could pull $\vx_t^\mathrm{E}$ towards any latent $\vx^\mathrm{E}_0$ that has a decoding that matches the measurements while the score function is pulling $\vx^\mathrm{E}_t$ towards a specific $\vx^\mathrm{E}_0$, i.e. towards $\E[\vx_0^\mathrm{E} | \vx^\mathrm{E}_t]$.

PSLD mitigates this problem by adding an additional term that pulls towards latents that are fixed points of the decoder-encoder map. Concretely, the approximation made in PSLD is:
\begin{gather}
    \nabla_{\vx^\mathrm{E}_t} \log p(\vy | \vX^\mathrm{E}_t=\vx^\mathrm{E}_t) \approx \nonumber \\ \nabla_{\vx^\mathrm{E}_t} \log p(\vy | (\vX_0 = \mathrm{Dec}(\E[\vX^\mathrm{E}_0 | \vX^\mathrm{E}_t=\vx^\mathrm{E}_t])) \nonumber \\ + \gamma_t \nabla_{\vx^\mathrm{E}_t} \left| \left| \E[\vx^\mathrm{E}_0 | \vx^\mathrm{E}_t] - \mathrm{Enc}( \mathrm{Dec}(\E[\vx^\mathrm{E}_0 | \vx^\mathrm{E}_t]) )\right|\right|^2,
    \label{eq:PSLD_approx}
\end{gather}
where $\gamma_t$ is a tunable parameter.

\subsubsection{Resample~\citet{song2024solving}}

Resample, a concurrent work with PSLD, proposes an alternative way to improve the performance of Latent DPS. After each clean prediction $\widehat{\vx}_0(\vx_{t+1}^{\mathrm{E}})$ is obtained from the previous sample $\vx_{t+1}^{\mathrm{E}}$ via Tweedie's formula in \Eqref{eq:tweedies}, and the unconditional reverse denoising process is updated using, say, DDIM:
\begin{align}
\vx_t':=\mathrm{UnconditionalDDIM}\big(\widehat{\vx}_0(\vx_{t+1}^{\mathrm{E}}), \vx_{t+1}^{\mathrm{E}}\big)
\end{align}
the authors project the latent back to a point $\widehat{\vx}_t$ that satisfies measurements using:
\begin{align}
    \mathcal{N} \Big(\widehat{\vx}_t;  \frac{\sigma_t^2\sqrt{\bar{\alpha}_t} \widehat{\vx}_0(\vy)+(1-\bar{\alpha}_t)\vx_t'}{\sigma_t^2 + (1-\bar{\alpha}_t)}, \frac{\sigma_t^2 (1-\bar{\alpha}_t)}{\sigma_t^2 + (1-\bar{\alpha}_t)}I_k \Big).
\end{align}
Here, $\sigma_t^2$ is a hyperparameter used to tune the alignment with measurements, $\bar{\alpha}_t$ is predefined in forward process, and $\widehat{\vx}_0(\vy)$ is found by solving:
\begin{align}
    \widehat{\vx}_0(\vy) \in \argmin_{\vx} \frac{1}{2}\norm{\vy - \mathcal{A}(\mathrm{Dec}(\vx))}_2 \quad \text{initialized at } \widehat{\vx}_0(\vx_{t+1}^{\mathrm{E}}).
\end{align}

\subsection{MPGD~\citet{hemanifold}}
The MPGD authors note that some methods require expensive computations for measurement alignment during gradient updates, as they involve passing through the gradient (chain rule) of the pre-trained diffusion model $\bm{\epsilon}_{\bm{\theta}}(\vx_t^{\mathrm{E}}, t)$: 
\begin{align}
    \vx_t^{\mathrm{E}} \leftarrow \vx_t^{\mathrm{E}} -\eta_t \nabla_{\vx_t^{\mathrm{E}}}\norm{\vy - \mathcal A\big(\mathrm{Dec}(\vx_{0\vert t})\big)}_2,
\end{align}
where $\vx_{0\vert t}:=\frac{1}{\sqrt{\bar{\alpha}_t}}\big(\vx_t^{\mathrm{E}} 
 - \sqrt{1-\bar{\alpha}_t} \bm{\epsilon}_{\bm{\theta}}(\vx_t^{\mathrm{E}}, t)\big)$ is  a clean estimation via Tweedie's formula in \Eqref{eq:tweedies}. This gradient bottleneck slows down the overall inverse problem solving. MPGD proposes bypassing the direct gradient $\nabla_{\vx_t^{\mathrm{E}}}$ with theoretical guarantees by updating with $\nabla_{\vx_{0 \vert t}}$:
\begin{align}
    \vx_{0\vert t}' \leftarrow \vx_{0\vert t} -\eta_t \nabla_{\vx_{0\vert t}}\norm{\vy - \mathcal A\big(\mathrm{Dec}(\vx_{0\vert t})\big)}_2
\end{align}
with 
\begin{align}
\vx_{0\vert t}:=\frac{1}{\sqrt{\bar{\alpha}_t}}\big(\vx_t^{\mathrm{E}} 
 - \sqrt{1-\bar{\alpha}_t} \bm{\epsilon}_{\bm{\theta}}(\vx_t^{\mathrm{E}}, t)\big),
\end{align}
and use the obtained $\vx_{0\vert t}'$ for unconditional reverse denoising process
\begin{align}
\vx_{t-1}':=\mathrm{UnconditionalDDIM}\big(\vx_{0\vert t}', \vx_{t}^{\mathrm{E}}\big).
\end{align}

\subsubsection{P2L~\citep{chung2024prompt}}

While text conditioning is a viable option for modern latent diffusion models such as Stable diffusion, the actual use was underexplored due to ambiguities on which text to use. P2L addresses this question by proposing an algorithm that optimizes for the text embedding on the fly while solving an inverse problem.
\begin{align}
\label{eq:p2l}
    \vc_t^{*} = \argmin_{\vc} \|\vy - A\mathrm{Dec}(\E[\vx^\mathrm{E}_0 | \vx^\mathrm{E}_t,\vc])\|^2,
\end{align}
where $\vc$ is the text embedding, and one can approximate $\E[\vx^\mathrm{E}_0 | \vx^\mathrm{E}_t,\vc]$ by using the Tweedie's formula with the denoiser conditioned on $\vc$. Using the optimized embedding at each timestep $\vc_t^*$, sampling follows the procedure of Latent DPS
\begin{align}
    \nabla_{\vx^\mathrm{E}_t} \log p(\vy | \vx^\mathrm{E}_t=\vx^\mathrm{E}_t, \vc) \approx \nabla_{\vx^\mathrm{E}_t} \log p(\vy | \vX_0 = \mathrm{Dec}(\E[\vx^\mathrm{E}_0 | \vx^\mathrm{E}_t=\vx^\mathrm{E}_t, \vc_t^*]))
\end{align}
In addition to the optimization of the text embedding, P2L further tries to leverage the VAE prior by decoding - running optimization in the pixel space - re-encoding
\begin{align}
\label{eq:p2l_xopt}
    \vx^* &= \argmin_{\vx} \|\vy - A\vx\|_2^2 + \lambda\|\vx - {\rm{Dec}}(\E[\vx^\mathrm{E}_0 | \vx^\mathrm{E}_t=\vx^\mathrm{E}_t])\|_2^2 \\
    \vx^\mathrm{E} &= {\rm{Enc}}(\vx^*)
\label{eq:p2l_zopt}
\end{align}

\subsubsection{TReg~\citep{kim2023regularization}, DreamSampler~\citep{kim2024dreamsampler}}
Instead of automatically finding a suitable text embedding to achieve maximal reconstructive performance, another advantage of text conditioning is that it can be used as an additional guiding signal to lead to a specific mode. This may seem trivial, as one has access to a conditional diffusion model. However, in practice, simply using a conditional diffusion model does not induce enough guidance as reported in~\citep{dhariwal2021diffusion,ho2022classifier}, and naively using classifier free guidance~\citep{ho2022classifier} (CFG) does not lead to satisfactory results. In addition to using data consistency imposing steps as in P2L, TReg proposes adaptive negation to update the null text embeddings used for CFG guidance.
\begin{align}
\label{eq:adaptive_negtaion}
    \vc_\varnothing^* = \argmin_{\vc} \mathrm{sim}(\mathcal{T}( \vx ^*), \vc),
\end{align}
where $\vx^*$ comes from \Eqref{eq:p2l_xopt}, $\rm{sim}$ denotes the CLIP similarity~\citep{clip} score, and $\gT$ is the CLIP image encoder. In essence, \Eqref{eq:adaptive_negtaion} minimizes the similarity between the current estimate of the image and the null text embedding. Hence, when the optimized $\vc_\varnothing$ is used for CFG with
\begin{align}
\label{eq:cfg_treg}
    \vepsilon_{\bm{\theta}}(\vx^\mathrm{E}_t, \vc_\varnothing^*) + \omega\left(
    \vepsilon_{\bm{\theta}}(\vx^\mathrm{E}_t, \vc) - \vepsilon_{\bm{\theta}}(\vx^\mathrm{E}_t, \vc_\varnothing^*),
    \right)
\end{align}
the conditioning vector direction $\vepsilon_{\bm{\theta}}(\vx^\mathrm{E}_t, \vc) - \vepsilon_{\bm{\theta}}(\vx^\mathrm{E}_t, \vc_\varnothing^*)$ is amplified. Later, TReg was further advanced by devising a way to better make use of CFG by combining score distillation sampling~\citet{poole2023dreamfusion} into the sampling framework.

\subsubsection{STSL~\citet{rout2024beyond}}
Most methods leverage the mean of the reverse diffusion distribution $p(\vX_0|\vX_t)$, and take a single gradient step with \Eqref{eq:latent_DPS_approx}. To further leverage the covariance of $p(\vX_0|\vX_t)$, Rout {\em et al.}~\citet{rout2024beyond} propose to use the following fidelity loss
\begin{align}
\label{eq:stsl_loss}
    \mathcal{L}(\vx^\mathrm{E}_t,\vy) = \nabla_{\vx^\mathrm{E}_t} \log p(\vy |\vX_0 = \mathrm{Dec}(\E[\vx^\mathrm{E}_0 | \vx^\mathrm{E}_t=\vx^\mathrm{E}_t]))  \nonumber \\ + \gamma \nabla_{\vx^\mathrm{E}_t} \left(
    {\rm Trace}\left(
    \nabla_{\vx^\mathrm{E}_t}^2 \log p(\vx^\mathrm{E}_t)
    \right)
    \right),
\end{align}
where $\gamma$ is a constant. To effectively compute the trace, one can further use the following approximation
\begin{align}
\label{eq:stsl_trace_approx}
    {\rm Trace}\left(
    \nabla_{\vx^\mathrm{E}_t}^2 \log p(\vx^\mathrm{E}_t)
    \right) \approx
    \E_{\bm{\epsilon} \sim \pi}\left[
    \bm{\epsilon}^\top\left(
    \nabla_{\vx^\mathrm{E}_t} \log p(\vx^\mathrm{E}_t + \bm{\epsilon}) - \nabla_{\vx^\mathrm{E}_t} \log p(\vx^\mathrm{E}_t)
    \right)
    \right],
\end{align}
where $\pi$ can be a Gaussian or a Rademacher distribution. Using the loss in \Eqref{eq:stsl_loss} with \Eqref{eq:stsl_trace_approx}, STSL uses multiple steps of stochastic gradient updates per timestep.

\section{Thoughts from the authors}
In the previous section, we presented several works in the space of using diffusion models to solve inverse problems. A natural question that both experts and newcomers to the field might have is, eventually,: ``\textit{which approach works the best?}''. Unfortunately, we cannot provide a conclusive answer to this question within the scope of this survey, but we can share a few thoughts. 

\paragraph{Thoughts about Explicit Approximations.} In this survey we tried to express seemingly very different works, such as DPS and DDRM, under a common mathematical language that contains the explicit approximations made for the measurements score. We observed that all the methods compute an error metric that matches consistency with the measurement and then lift the error back to the image space dimensions to perform the gradient update. Some of the methods used noised versions of the measurements to compute the error while others use the clean measurements. To the best of our knowledge, it is not clear which one works the best and one can derive new approximation algorithms by simply making the dual change to any of the methods that already exist, e.g. one can propose Score-ALD++ by using the noisy measurements to compute the error. By looking at Figure~\ref{fig:tree}, it is also evident that methods propose increasingly more complex ``lifting'' matrices. Some of these approximations require increased computation, e.g. the Moments Matching method. We strongly believe that the field would benefit from a standardized benchmark for diffusion models and inverse problems to understand better the computational performance trade-offs of different methods. We also believe that under certain distributional assumptions, it should be possible to characterize analytically the propagation of the approximation errors induced by the different methods.

\paragraph{Thoughts about Variational Methods.} Variational Methods try to estimate the parameters of a simpler distribution. The benefit here is that one can employ well-known optimization techniques to better solve the optimization problem at hand. A potential drawback of this approach is that the proposed distribution might not be able to capture the complexity of the real posterior distribution.

\paragraph{Thoughts about CSGM-type Methods.} CSGM-type frameworks can benefit from the plethora of techniques that have been previously developed to solve inverse problems with GANs and other deep generative modeling frameworks. The main issue here is computational since the generative model to be inverted here is the Probability Flow ODE mapping that requires several calls to the diffusion model. Consistency Models~\citep{song2023consistency,kim2023consistency} and other approaches such as Intermediate Layer Optimization could mitigate this issue.

\paragraph{Thoughts about Asymptotically Exact Methods.} Asymptotically Exact Methods, usually based on Monte Carlo, could be useful when sampling from the true posterior is really important. However, the theoretical guarantees of these methods only hold under the setting of infinite computation and it remains to be seen if they can scale to more practical settings.

\section{Conclusion}
In this survey, we discussed different types of inverse problems and different approaches that have been developed to solve them using diffusion priors. We identified four distinct families: methods that propose explicit approximations for the measurement score, variational inference methods, CSGM-type frameworks and finally approaches that asymptotically guarantee exact sampling (at the cost of increased computation). The different frameworks and the works therein are all trying to address the fundamental problem of the intractability of the posterior distribution. In this survey, we tried to unify seemingly different approaches and explain the trade-offs of different methods. We hope that this survey will serve as a reference point for the vibrant field of diffusion models for inverse problems.

\section*{Acknowledgments}
This research has been supported by NSF Grants AF 1901292, CNS 2148141, Tripods CCF 1934932, IFML CCF 2019844 and research gifts by Western Digital, Amazon, WNCG IAP, UT Austin Machine Learning Lab (MLL), Cisco and the Stanly P. Finch Centennial Professorship in Engineering. Giannis Daras has been supported by the Onassis Fellowship (Scholarship ID: F ZS 012-1/2022-2023), the Bodossaki Fellowship and the Leventis Fellowship. The authors would like to thank our colleagues Viraj Shah, Miki Rubinstein, Murata Naoki, Yutong He, and Stefano Ermon for helpful discussions.


\pagebreak
\appendix

\section{Proofs}

\begin{lemma}[Conditional Expectation and MMSE]
Let $\vX_0$ and $\vX_t$ be two random variables, and $h_{\bm{\theta}}(\vx_t, t)$ be a function parameterized by $\bm{\theta}$. Then:

\begin{gather}
   \textrm{argmin}_{\bm{\theta}} \mathbb{E}\left[\left|\left||h_{\bm{\theta}}(\vx_t, t) - \vx_0\right|\right|^2\right] = \textrm{argmin}_{\bm{\theta}} \mathbb{E}\left[\left|\left|h_{\bm{\theta}}(\vx_t, t) - \mathbb{E}[\vx_0|\vx_t]\right|\right|^2\right]
\end{gather}
That is, the function $h_{\bm{\theta}}(\vx_t, t)$ that minimizes the mean squared error with respect to $\vx_0$ is the one that best approximates the conditional expectation $\mathbb{E}[\vx_0|\vx_t]$.
\label{lemma:mmse_exp}
\end{lemma}

\begin{proof}
\begin{gather}
    \textrm{argmin}_{{\bm{\theta}}} \mathbb E\left[ \left| \left| \vh_{{\bm{\theta}}}(\vx_t, t) - \vx_0\right|\right|^2\right] \\
    = \textrm{argmin}_{{\bm{\theta}}} \mathbb E\left[ \left| \left| \vh_{{\bm{\theta}}}(\vx_t, t) - \mathbb E[\vx_0 |\vx_t] + \mathbb E[\vx_0 |\vx_t] - \vx_0\right|\right|^2\right] \\
    = \textrm{argmin}_{{\bm{\theta}}} \mathbb E\bigg[ \left| \left| \vh_{{\bm{\theta}}}(\vx_t, t) - \mathbb E[\vx_0 |\vx_t]\right|\right|^2 - 2(\vh_{{\bm{\theta}}}(\vx_t, t) - \mathbb E[\vx_0 |\vx_t] )^{\top}(\vx_0  - \mathbb E[\vx_0 |\vx_t]) \nonumber \\ + ||\vx_0 - \mathbb E[\vx_0 |\vx_t]||^2\bigg] \\
    = \textrm{argmin}_{{\bm{\theta}}} \mathbb E\left[ \left| \left| h_{{\bm{\theta}}}(\vx_t, t)- \mathbb E[\vx_0 |\vx_t]\right|\right|^2 - 2h_{{\bm{\theta}}}(\vx_t, t)^{\top}(\vx_0 - \mathbb E[\vx_0 |\vx_t])\right].
\end{gather}
Now, for the second term, we have:
\begin{gather}
    \mathbb E_{\vx_0, \vx_t}\left[ \vh_{{\bm{\theta}}}(\vx_t, t)^{\top} (\vx_0 - \E[\vx_0 | \vx_t])\right] = \mathbb E_{\vx_t} \mathbb E_{\vx_0 | \vx_t} \left[ \vh_{{\bm{\theta}}}(\vx_t, t)^{\top} (\vx_0 - \E[\vx_0 | \vx_t])\right] \\ 
    =  \mathbb \E_{\vx_t} \left[\vh_{{\bm{\theta}}}(\vx_t, t)^{\top} \left(\mathbb \E_{\vx_0 | \vx_t} \left[ (\vx_0 - \E[\vx_0 | \vx_t])\right]\right)\right] = 0,
\end{gather}
which concludes the proof.
\end{proof}

\subsection{Tweedie's Formula}
\begin{lemma}[Tweedie's Formula]
\label{lemma:Tweedies}
    Let:
    \begin{gather}
        \vX_t = \vX_0 + \sigma_t \vZ,
    \end{gather}
    for $\vX_0 \sim p_{\vX_0}$ and $\vZ \sim \mathcal N(0, I)$.
    Then,
    \begin{gather}
        \nabla_{\vx_t} \log p_t(\vx_t) = \frac{\E[\vX_0 | \vx_t] - \vx_t}{\sigma_t^2}.
    \end{gather}
\end{lemma}

\begin{proof}
    \begin{gather}
        \nabla_{\vx_t} \log p_t(\vx_t) = \frac{1}{p_t(\vx_t)}\nabla_{\vx_t} p_t(\vx_t) = \frac{1}{p_t(\vx_t)} \nabla_{\vx_t} \int p_t(\vx_t, \vx_0)\dx_0 \\
        = \frac{1}{p_t(\vx_t)} \nabla_{\vx_t} \int p_t(\vx_t | \vx_0) p_0(\vx_0)\dx_0 \\
        =  \frac{1}{p_t(\vx_t)}  \int \nabla_{\vx_t} p_t(\vx_t | \vx_0) p_0(\vx_0)\dx_0 \\
        = \frac{1}{p_t(\vx_t)}  \int p_t(\vx_t | \vx_0) \nabla_{\vx_t} \log p_t(\vx_t | \vx_0) p_0(\vx_0)\dx_0 \\
        =  \int p_0(\vx_0 | \vx_t) \frac{\vx_0 - \vx_t}{\sigma_t^2} \dx_0 \\
        = \frac{\E[\vX_0 | \vx_t] - \vx_t}{\sigma_t^2}.
    \end{gather}
\end{proof}

\subsection{Denoising Score Matching}

By leveraging the MMSE interpretation of the conditional expectation and Tweedie's formula, one can approximate the score function by training a model to predict the clean image from a corrupted observation (via supervised learning). At inference time, the trained network can be converted to a model that approximates the score through Tweedie's formula. This training procedure is typically known as $\vx_0$-prediction loss. An alternative, but equivalent, way is to train for the score directly. \citet{vincent2011connection} independently discovered Denoising Score Matching, which has as a unique minimizer the score function. DSM and the $\vx_0$-prediction objective are the same up to a simple network reparametrization.

\begin{theorem}[Denoising Score Matching~\citep{vincent2011connection}]
Let $p_0, p_t$ be two distributions in $\mathbb R^n$. Assume that all the conditional distributions, $p_t(\vx_t|\vx_0)$, are supported and differentiable in $\mathbb R^n$. Let:
\begin{gather}
    J_1(\theta) = \frac{1}{2}\mathbb E_{\vx_t \sim p_t}\left[\left|\left| \vs_{\theta}(\vx_t) - \nabla_{\vx_t} \log p_t(\vx_t)\right|\right|^2\right],
\end{gather}
\begin{gather}
    J_2(\theta) = \frac{1}{2}\mathbb E_{(\vx_0, \vx_t) \sim p_0(\vx_0) p_t(\vx_t|\vx_0)}\left[\left| \left| \vs_{\theta}(\vx_t) - \nabla_{\vx_t}\log p_t(\vx_t|\vx_0) \right|\right|^2\right].
\end{gather}
Then, $J_1$ and $J_2$ have the same minimizer.
\label{theorem:dsm}
\end{theorem}

We include the proof listed in \citep{daras2023soft} for completeness. 
\begin{proof}
\begin{gather}
    J_1(\theta) = \frac{1}{2}\mathbb E_{\vx_t \sim p_t}\left[\left|\left| \vs_{\theta}(\vx_t)\right| \right|^2 - 2 \vs_{\theta}(\vx_t)^{\top}\nabla_{\vx_t}\log p_t(\vx_t) + ||\nabla_{\vx_t} \log p_t(\vx_t)||^2\right] \\
    = \frac{1}{2} \mathbb E_{\vx_t\sim p_t}\left[ || \vs_{\theta}(\vx_t)||^2\right] -  \mathbb E_{\vx_t\sim p_t}\left[s_{\theta}(\vx_t)^{\top}\nabla_{\vx_t}\log p_t(\vx_t)\right] + C_1.
\end{gather}

Similarly,
\begin{gather}
    J_2(\theta) = \frac{1}{2} \mathbb E_{\vx_t\sim p_t}\left[ || \vs_{\theta}(\vx_t)||^2\right] -   \mathbb E_{(\vx_0, \vx_t)} \left[\vs_{\theta}(\vx_t)^{\top}\nabla_{\vx_t}\log p_t(\vx_t|\vx_0)\right] + C_2.
\end{gather}

It suffices to show that:
\begin{gather}
    \mathbb E_{\vx_t\sim p_t}\left[ s_{\theta}(\vx_t)^{\top}\nabla_{\vx_t}\log p_t(\vx_t)\right] \nonumber \\ = \mathbb E_{(\vx_0, \vx_t) \sim p_0(\vx_0) p_t(\vx_t|\vx_0)} \left[s_{\theta}(\vx_t)^{\top}\nabla_{\vx_t}\log p_t(\vx_t|\vx_0)\right].
\end{gather}

We start with the second term.
\begin{gather}
    \mathbb E_{(\vx_0, \vx_t) \sim p_0(\vx_0) p_t(\vx_t|\vx_0)} \left[\vs_{\theta}(\vx_t)^{\top}\nabla_{\vx_t}\log p_t(\vx_t|\vx_0)\right] \nonumber \\ = \int_{\vx_0}\int_{\vx_t}p_0(\vx_0) p_t(\vx_t|\vx_0) \vs_{\theta}(\vx_t)^{\top} \nabla_{\vx_t}\log p_t(\vx_t|\vx_0)\mathrm{d}{\vx_t}\mathrm{d}{\vx_0} \\
    = \int_{\vx_0}\int_{\vx_t} \vs_{\theta}^{\top}(\vx_t) \left( p_0(\vx_0) p_t(\vx_t|\vx_0) \nabla_{\vx_t}\log p_t(\vx_t|\vx_0)\right)\mathrm{d}{\vx_t}\mathrm{d}{\vx_0} \\
    =  \int_{\vx_0}\int_{\vx_t}\vs_{\theta}^{\top}(\vx_t) \left( p_0(\vx_0) p_t(\vx_t|\vx_0) \frac{1}{p_t(\vx_t|\vx_0)} \nabla_{\vx_t} p_t(\vx_t|\vx_0)\right)\mathrm{d}{\vx_t}\mathrm{d}{\vx_0} \\
    = \int_{\vx_0}\int_{\vx_t} \vs_{\theta}^{\top}(\vx_t) \left( p_0(\vx_0)  \nabla_{\vx_t} p_t(\vx_t|\vx_0)\right)\mathrm{d}{\vx_t}\mathrm{d}{\vx_0} \\ 
    = \int_{\vx_t}\int_{\vx_0} \vs_{\theta}^{\top}(\vx_t) \left( p_0(\vx_0)  \nabla_{\vx_t} p_t(\vx_t|\vx_0)\right)\mathrm{d}{\vx_0}\mathrm{d}{\vx_t} \\ 
    = \int_{\vx_t} \vs_{\theta}^{\top}(\vx_t) \left(\int_{\vx_0}p_0(\vx_0) \nabla_{\vx_t} p_t(\vx_t|\vx_0) \mathrm{d}{\vx_0}\right)\mathrm{d}{\vx_t} \\ 
    = \int_{\vx_t}\vs_{\theta}^{\top}(\vx_t) \left(\int_{\vx_0} \nabla_{\vx_t} \left(p_0(\vx_0)  p_t(\vx_t|\vx_0)\right) \mathrm{d}{\vx_0}\right)\mathrm{d}{\vx_t} \\
    = \int_{\vx_t}\vs_{\theta}^{\top}(x_t) \left(\nabla_{\vx_t}\left(\int_{\vx_0} p_0(\vx_0) p_t(\vx_t|\vx_0) \mathrm{d}{\vx_0}\right)\right)\mathrm{d}{\vx_t} \\
    = \int_{\vx_t}\vs_{\theta}^{\top}(\vx_t) \nabla_{\vx_t}p_t(\vx_t)\mathrm{d}{\vx_t} \\ 
    = \int_{\vx_t}p_t(\vx_t) \vs_{\theta}^{\top}(\vx_t) \nabla_{\vx_t}\log p_t(\vx_t)\mathrm{d}{\vx_t} \\
    = \mathbb E_{\vx_t \sim p_t(\vx_t)}\left[ \vs_{\theta}^{\top}(\vx_t)\nabla_{\vx_t}\log p_t(\vx_t)\right].
\end{gather}
\end{proof}

\renewcommand{\H}{\mathrm{H}}
\newcommand{\J}{\mathrm{Jacob}}

\subsection{Jacobian of the score}
\begin{lemma}[Jacobian of score-function]
    Let:
    \begin{gather}
        \vX_t = \vX_0 + \sigma_t \vZ,
    \end{gather}
    for $\vX_0 \sim p_{\vX_0}$ and $\vZ \sim \mathcal N(0, I)$.
    Then,
    \begin{gather}
        \H(\log p_{\vX_t})(\vx_t) = \frac{\E[\vX_0\vX_0^{\top} | \vx_t] -\E[\vX_0 |\vx_t]\E^{\top}[\vX_0 |\vx_t]}{\sigma_t^4} - \frac{I}{\sigma_t^2}.
        \label{eq:score_jacobian}
    \end{gather}
    \label{lemma:score_jacobian}
\end{lemma}

\begin{proof}

\begin{gather}
    \nabla_{\vx_t} \log p_{\vX_t}(\vx_t) = \frac{\E[\vX_0 |\vx_t] - \vx_t}{\sigma_t^2} \\ 
    \Rightarrow \sigma_t^2 \H(\log p_{\vX_t})(\vx_t) = \J\left(\E[\vX_0|\vx_t] \right) - I.
\end{gather}
We will now analyze the Jacobian.
\begin{gather}
    \J\left(\E[\vX_0|\vx_t] \right) = \int_{} \nabla_{\vx_t}\left(p_{\vX_0}(\vx_0|\vx_t)\vx_0\right)\dx_0 \\
    = \int_{} \vx_0\nabla_{\vx_t}^{\top} p_{\vX_0}(\vx_0|\vx_t)\dx_0 \\ 
    = \int_{} \vx_0 p_{\vX_0}(\vx_0 |\vx_t) \nabla_{\vx_t}^{\top} \log\left(  \frac{p_{\vX_t}(\vx_t | \vx_0) p_{\vX_0}(\vx_0)}{p_{\vX_t}(\vx_t)} \right)  \dx_0 \\
    = \int_{} \vx_0 p_{\vX_0}(\vx_0 |\vx_t) \nabla_{\vx_t}^{\top} \log\left(  \frac{p_{\vX_t}(\vx_t | \vx_0)}{p_{\vX_t}(\vx_t)} \right)  \dx_0 \\
    = \int_{} \vx_0 p_{\vX_0}(\vx_0 |\vx_t) \nabla_{\vx_t}^{\top} \log p_{\vX_t}(\vx_t | \vx_0)  \dx_0 - \int_{} \vx_0 p_{\vX_0}(\vx_0 |\vx_t) \nabla_{\vx_t}^{\top} \log p_{\vX_t}(\vx_t)  \dx_0 \\
    = \int_{} \vx_0 p_{\vX_0}(\vx_0 |\vx_t) \frac{\vx_0^{\top} - \vx_t^{\top}}{\sigma_t^2} \dx_0 - \int_{} \vx_0 p_{\vX_0}(\vx_0 |\vx_t) \nabla_{\vx_t}^{\top} \log p_{\vX_t}(\vx_t)    \dx_0 \\
    = \int_{} \vx_0 p_{\vX_0}(\vx_0 |\vx_t) \frac{\vx_0^{\top} - \vx_t^{\top}}{\sigma_t^2} \dx_0 - \int_{} \vx_0 p_{\vX_0}(\vx_0 |\vx_t)  \frac{\E^{\top}[\vx_0 | \vx_t] - \vx_t^{\top}}{\sigma_t^2}   \dx_0 \\
    = \frac{1}{\sigma_t^2}\left( \E[\vx_0\vx_0^{\top} | \vx_t] - \E[\vx_0 | \vx_t]\E[\vx_0|\vx_t]^{\top}\right).
\end{gather}

\end{proof}

\begin{corollary}
    Let:
    \begin{gather}
        \vX_t = \vX_0 + \sigma_t \vZ,
    \end{gather}
    for $\vX_0 \sim p_{\vX_0}, \ \vX_0 \in \R^n$ and $\vZ \sim \mathcal N(0, I)$.
    Then,
    \begin{gather}
        \nabla_{\vx_t}^2 \log p_t(\vx_t) = \frac{\E[||\vX_0||^2 \ | \ \vx_t] - \norm{\E[\vX_0|\vx_t]}}{\sigma_t^4} - \frac{n}{\sigma_t^2}.
        \label{eq:score_laplacian}
    \end{gather}
    \label{corr:score_laplacian}
\end{corollary}

\bibliographystyle{IEEEtranN}
\bibliography{references}

\end{document}